%% file: main.tex
\documentclass{article}

\usepackage{PRIMEarxiv}

\usepackage[utf8]{inputenc} 
\usepackage[T1]{fontenc}    
\usepackage{amsfonts}       
\usepackage{nicefrac}       
\usepackage{microtype}      
\usepackage{lipsum}
\usepackage{fancyhdr}       
\usepackage{graphicx}       
\graphicspath{{media/}}     

\usepackage{graphicx}
\usepackage{stfloats}
\usepackage[round]{natbib}

\usepackage[colorlinks=true, linkcolor=red, citecolor=blue, urlcolor=green]{hyperref}      
\usepackage{booktabs}       
\usepackage{amsfonts}       
\usepackage{nicefrac}       
\usepackage{microtype}      
\usepackage{xcolor}         

\usepackage{graphicx}
\usepackage{amsmath}

\usepackage{algorithmic}
\usepackage{subfigure}
\usepackage{float,psfrag,epsfig,color,xcolor,url}
\usepackage{diagbox}
\makeatother
\newtheorem{assumption}{Assumption}
\usepackage{algorithm}
\usepackage{bbm}
\usepackage{algorithmic}
\newtheorem{definition}{Definition}
\newtheorem{theorem}{Theorem}
\newtheorem{lemma}{Lemma}
\newtheorem{corollary}{Corollary}
\newtheorem{proposition}{Proposition}
\newtheorem{proof}{Proof}[section]
\usepackage{float}
\title{Understanding Inverse Reinforcement Learning under Overparameterization: Non-Asymptotic Analysis and Global Optimality
}

\author{
  Ruijia Zhang \\
  Johns Hopkins University \\
 {rzhan127@jh.edu} \\
  \And 
  Siliang Zeng \\
 University of Minnesota\\
  zeng0176@umn.edu \\
   \And
  Chenliang Li \quad Alfredo Garcia \\
  Texas A\&M University \\
\texttt{\{chenliangli,alfredo.garcia\}}@tamu.edu \\
  \And
  Mingyi Hong \\
  University of Minnesota \\
  \texttt{mhong@umn.edu} \\
}

\begin{document}
\maketitle

\begin{abstract}
The goal of the Inverse reinforcement learning (IRL) task is to identify the underlying
reward function and the corresponding optimal policy from a set of expert demonstrations. While most IRL algorithms'
theoretical guarantees rely on a linear reward
structure, we aim to extend the theoretical understanding of IRL to scenarios where the reward function is parameterized by neural networks. Meanwhile, conventional IRL algorithms usually adopt a nested structure, leading to computational inefficiency, especially in high-dimensional settings.
To address this problem, we propose the first two-timescale single-loop IRL algorithm under neural network parameterized reward and provide a non-asymptotic convergence analysis under overparameterization. Although prior optimality results for linear rewards do not apply, we show that our algorithm can identify the globally optimal reward and policy under certain neural network structures. This is the first IRL algorithm with a non-asymptotic convergence guarantee that provably achieves global optimality in neural network settings.
\end{abstract}
 
\section{Introduction}
Given the observed trajectories of states and actions executed by an expert, we consider the challenge of inferring the reinforcement learning reward function in which the expert was trained.
This problem is commonly known as inverse reinforcement learning (IRL), as discussed in a recent survey by \cite{Osa_2018}. IRL involves the task of estimating both the reward function and the expert's policy that best aligns with the provided data. While there are certain limitations on the identifiability of rewards \citep{kim2021reward}, the process of estimating rewards from expert trajectories offers valuable insights for estimating optimal policies under different environment dynamics and/or learning new tasks through reinforcement learning, which has been widely used in different applications such as autonomous driving \citep{Kurach2018TheGL,neu2012apprenticeshiplearningusinginverse,phanminh2022drivingreallifeinverse}, autonomous trading \citep{yang2018investor}, and LLM alignment \citep{ouyang2022traininglanguagemodelsfollow}. 

In the Maximum Entropy Inverse Reinforcement Learning (MaxEnt-IRL) formulation introduced by \cite{ziebart2008maximum}, the expert's behavior is modeled as a policy that maximizes entropy while satisfying a constraint to ensure that the expected features align with the empirical averages in the expert's dataset. The algorithms developed for MaxEnt-IRL \citep{ziebart2008maximum, Ziebart2010ModelingIV,wulfmeier2015maximum} usually employ a nested loop structure. These algorithms cycle through two main stages: an outer loop focused on the reward update, and an inner loop responsible for calculating precise policy estimates.
Although these nested loop algorithms remain computationally feasible within tabular contexts, they pose significant computational challenges in high-dimensional settings where function approximation is necessary \citep{finn2016guided}.

Most recently, a new formulation of IRL based on maximum likelihood (ML) estimation has been proposed \citep{zeng2022maximum}. The ML formulation is a bi-level optimization problem. In this framework, the upper-level problem maximizes the log-likelihood function given observed trajectories, while the lower-level problem seeks to identify the optimal policy given the current reward parameterization. To alleviate the computational burden associated with the recurrent optimization of the lower-level policy, ML-IRL carries out the policy improvement step and reward update alternately. The new policy at each round is generated by \textit{soft policy iteration} \citep{haarnoja2017reinforcement}. This approach ensures that each iteration can be executed with relatively lower computational costs. 

However, whether ML-IRL converges or not relies on its policy improvement step, which has not been thoroughly discussed before. An $\epsilon-$accurate estimation of soft Q-function for policy improvement is assumed to ensure global convergence. It's worth noting that this assumption remains intractable in practice. The single-loop implementation of ML-IRL \citep{zeng2022maximum} carries out multiple soft Actor-Critic (SAC) steps for the policy update instead. It therefore lacks a convergence guarantee to substantiate its performance.

Moreover, existing ML-IRL works only consider the optimality analysis when the reward employs a linear structure. \cite{zeng2023understanding} establishes a strong duality relationship with maximum entropy Inverse Reinforcement Learning (IRL) when the reward can be expressed as a linear combination of features. 

We want to advance IRL research by developing a theoretical understanding of IRL under neural network parameterization. 

\subsection{Challenges:}
The main challenge is how to ensure the global convergence of our IRL algorithm when we don't have an accurate policy evaluation in practice. Besides, the nonconvex structure of neural network parametrization, coupled with bi-level optimization formulation, complicates the optimality analysis. 

Therefore, this paper seeks to address these challenges and answer the following questions: \\
\textit{(i) Can we design an IRL framework that ensures convergence without requiring precise policy evaluation, even when the reward function is parameterized by neural networks? \\
(ii) Is it possible to develop a computationally efficient single-loop algorithm within this framework that avoids inner policy optimization, while still maintaining convergence?\\
(iii) Moreover, does this algorithm provably identify the ground truth reward function and the corresponding optimal policy?}

In this work, we provide affirmative answers to the above questions. We summarize our major contributions here.

\subsection{Contributions:}
\begin{itemize}
    \item We design the first theoretically guaranteed ML-IRL framework under neural network settings. This framework comprises two main stages: (i) In the policy improvement stage, we apply a finite-step neural soft Q-learning \citep{haarnoja2017reinforcement,cai2023neural} to update the policy.  (ii) In the reward update stage, we first estimate the gradient of the maximum log-likelihood objective function with respect
to the reward parameter through the sampled trajectory induced by the present policy. Subsequently, gradient ascent is employed to update the reward parameter based on this gradient estimate. 

\item To reduce the computational complexity of nested soft Q-learning in the policy improvement stage, we propose a single-loop ML-IRL algorithm for solving the bi-level IRL problem. Inspired by the two-timescale alternating update process of actor-critic \citep{hong2020two, borkar1997stochastic}, our approach carefully differentiates the stepsizes for Temporal Difference (TD) updates and reward parameter updates, which ensures global convergence with just one iteration of soft Q-learning. 

\item Additionally, we show that the algorithm has strong theoretical guarantees: it requires $\mathcal{O}\left(\epsilon^{-8}\right)$ steps to identify an $\epsilon$-approximate globally optimal reward estimator under overparameterization. To our knowledge, this is the first single-loop algorithm with finite-time convergence analysis and global optimality guarantees for IRL problems parameterized by neural networks.

\item  Finally, we conduct intensive numerical experiments on Mujoco tasks. We show that our methods have better performances than many previous state-of-the-art IRL algorithms for imitating the expert behavior.

\end{itemize}

\section{Related Work}

\subsection{Single-loop IRL Algorithms}
Towards more efficient training, some recent works 
have developed algorithms to alleviate the computational burden of nested-loop IRL structures. In \cite{reddy2019sqil}, the authors attempt to model the IRL using
certain maximum entropy RL problem with a very trivial reward
function that assigns $r=+1$ for matching expert demonstrations and $r=0$ for all other behaviors. \cite{garg2021iq} proposes to transform the
standard formulation of IRL into a single-level problem by directly estimating the soft Q-function, which encodes the information of the reward function and policy. However, the implicit
reward function recovered by the soft Q-function is not very accurate since it is not strictly subject to satisfy
Bellman’s equation. Another approach called f-IRL \citep{ni2021f} estimates rewards based
on the minimization of several measures of divergence with respect to the expert’s state visitation
measure. However, it is restricted to certain scenarios where the reward function only depends on the state. Besides, no convergence guarantee is given to support its performance on the reported numerical results of the single-loop implementation. 

\subsection{Convergence Analysis of IRL algorithms}
\cite{syed2007game,syed2008apprenticeship,abbeel2004apprenticeship,neu2012apprenticeshiplearningusinginverse, ross2010efficient,rajaraman2020fundamentallimitsimitationlearning} first study the convergence of IRL but only in the tabular case. The convergence result of Behavior Cloning (BC) established by \cite{rashidinejad2023bridgingofflinereinforcementlearning} fails in continuous state and action spaces with a horizon $H\geq2$,  since BC is considered as
a classification problem and always faces unseen states in the continuous state space. \cite{liu2021provably} and \cite{zhang2020generative} present a convergence guarantee for GAIL in the linear function and neural network approximation setting respectively. However, the problem formulation of GAIL is a min-max optimization, where the inner
part maximizes the reward and the outer part minimizes the policy. \cite{liu2021provably} and \cite{zhang2020generative} only
show that the learned policy cannot be distinguished from the expert policy in expectation, without mentioning the recovery of the expert reward function.


\subsection{Neural Network Analysis}
Our theoretical understanding of optimality guarantees is inspired by the literature on the local linearization property of the neural network parameterization \citep{maei2009convergent,bertsekas2018feature,geist2013algorithmic}. 
Previous works on the optimality analysis of neural networks such as \cite{li2018learning,du2018power} show that for a one-hidden-layer network with ReLU activation function using overparameterization and random initialization, GD and SGD can find the near global-optimal solutions in polynomial time.  \cite{zou2018stochastic,allenzhu2020learninggeneralizationoverparameterizedneural,gao2019convergence} extends the result to $L$-hidden-layer-fully-connected neural network structure. Several recent works \cite{cai2023neural,gaur2023global,fu2021singletimescale,zhang2020generative} discusses optimality results in deep reinforcement learning. However, previous results cannot be easily generalized to the bi-level IRL optimization formulation, which is significantly more challenging. 


\section{Preliminaries}
In RL, a discounted Markov decision process (MDP) is denoted by a tuple $\mathcal{M}=(\mathcal{S}, \mathcal{A}, P, r, \mu_0, \gamma)$, where $\mathcal{S}$ is the state space, $\mathcal{A}$ is the action space, $P(\cdot \mid s, a)$ is the transition probability kernel, $r: \mathcal{S} \times \mathcal{A} \rightarrow \mathbb{R}$ is the reward function, $\mu_0(\cdot)$  is the initial state distribution and $\gamma$ is the discount factor. A policy $\pi$ takes state $s \in \mathcal{S}$ as an input and gives a distribution over actions $\mathcal{A}$. $\mu_\pi$ is the stationary state-action distribution associated with the policy $\pi$. We consider $\mathcal{S}$ to be continuous and $\mathcal{A}$ to be finite.

\subsection{{ Entropy-Regularized Reinforcement Learning}} 
Maximum entropy RL \citep{haarnoja2017reinforcement} searches the optimal policy, which maximizes its entropy at each visited state:
\begin{equation}
\label{eq:maxen}
    \pi_{\text {MaxEnt }}^*=\arg \max _\pi\sum_{t=0}^{\infty}\mathbb{E}_{\left(\mathbf{s}_t, \mathbf{a}_t\right) \sim \mu_\pi}\left[\gamma^t \left({r}\left(\mathbf{s}_t, \mathbf{a}_t\right)+\mathcal{H}\left(\pi\left(\cdot \mid \mathbf{s}_t\right)\right)\right)\right],
\end{equation}
where $\mathcal{H}(\pi(\cdot \mid s)):=$ $-\sum_{a \in \mathcal{A}} \pi(a \mid s) \log \pi(a \mid s)$ denotes the entropy of policy $\pi(\cdot \mid s)$.
In this context, the soft V-function and soft Q-function are defined as follows:
\begin{align}
    &V^{\text{soft}}_{r, \pi} (s) = \mathbb{E}_{\pi, s_0 = s} \left[ \sum_{t=0}^{\infty} \gamma^t \left( r(s_t, a_t) + \mathcal{H} \left( \pi(\cdot \mid s_t) \right) \right) \right], \label{2a} \\
    &Q^{\text{soft}}_{r, \pi} (s, a) = r(s, a) + \gamma \mathbb{E}_{s' \sim \mathcal{P}(\cdot \mid s, a)} \left[ V^{\text{soft}}_{r, \pi} (s') \right]. \label{2b}
\end{align}

By soft policy iteration \citep{haarnoja2017reinforcement}, a greedy entropy-regularized policy at $(k+1)$th iteration $\pi_{k+1}$ can be generated by $\pi_{k}$ at $k$th step as follows:
$\pi_{k+1}(\cdot \mid s) \propto \exp \left(Q_{r, \pi_k}^{s o f t}(s, \cdot)\right), \forall s \in \mathcal{S}$ \citep{Ziebart2010ModelingIV,haarnoja2017reinforcement}. Under a fixed reward function, it can be shown that the new policy $\pi_{k+1}$ monotonically improves $\pi_k$, and it converges linearly to the optimal policy of \eqref{eq:maxen}; see Theorem 4 in \cite{haarnoja2017reinforcement} and Theorem 1 in \cite{cen2022fast}.


\subsection{Problem Formulation: Maximum Log-Likelihood IRL (ML-IRL)}


In this section, we review the fundamentals of maximum
likelihood inverse reinforcement learning (ML-IRL). 

Assume observations are in the form of expert state-action trajectories $\tau^{E}=\{\left(s_t,a_t\right)\}_{t\geq0}$ drawn from an expert policy.
A model of the expert's behavior under parameterized reward 
 $ \widehat{r}(s,a;\theta)$ is a randomized policy $\pi_\theta(\cdot \mid s)$ where $\theta$ is a parameter vector. Assuming the transition dynamics $\mathcal{P}\left(s_{t+1} \mid s_t, a_t\right)$ are known, the discounted log-likelihood of observing a sample trajectory $\tau^E$ under model $\pi_\theta$ can be written as follows:
\begin{equation}
 \mathbb{E}_{\tau^E \sim \pi^E}\left[\log \prod_{t \geq 0}\left(\mathcal{P}\left(s_{t+1} \mid s_t, a_t\right) \pi_\theta\left(a_t \mid s_t\right)\right)^{\gamma^t}\right]=\mathbb{E}_{\tau^E \sim \pi^E}\left[\sum_{t \geq 0} \gamma^t\left(\log \pi_\theta\left(a_t \mid s_t\right)+\log \mathcal{P}\left(s_{t+1} \mid s_t, a_t\right)\right)\right].
\end{equation}

Let $\label{L_theta}\mathcal{L}(\theta):=\mathbb{E}_{\tau^E \sim \pi^E}\left[\sum_{t>0} \gamma^t \log \pi_\theta\left(a_t \mid s_t\right)\right]$. The IRL problem can then be formulated as the following maximum log-likelihood IRL formulation \citep{zeng2022maximum}:
\begin{equation}
\label{formulation}
\begin{aligned}
        &\quad \max_\theta ~ \mathcal{L}(\theta) \\
        &
\text {s.t.} ~ \pi_\theta=\arg \max _\pi \mathbb{E}_\pi\left[\sum_{t=0}^{\infty} \gamma^t\left(\widehat{r}\left(s_t, a_t ; \theta\right)+\mathcal{H}\left(\pi\left(\cdot \mid s_t\right)\right)\right)\right],
    \end{aligned}
\end{equation}
where $\mathcal{H}(\pi(\cdot \mid s)):=$ $-\sum_{a \in \mathcal{A}} \pi(a \mid s) \log \pi(a \mid s)$ denotes the entropy of policy $\pi(\cdot \mid s)$. 

However, in practice, the ground truth policy $\pi^E$ is unknown. Let $\mathcal{D}:=\left\{\tau^E_i\right\}_{i=1}^N$ denote the demonstration dataset containing only finite observed trajectories. The empirical discounted log-likelihood $\widehat{\mathcal{L}}(\theta)$  is denoted by $\mathbb{E}_{\tau^E \sim \mathcal{D}}\left[\sum_{t>0} \gamma^t \log \pi_\theta\left(a_t \mid s_t\right)\right]$. Therefore, we consider an empirical version of \eqref{formulation} as follows.

\begin{equation}
\label{empirical formulation}
\begin{aligned}
        &\quad \max_\theta ~ \widehat{\mathcal{L}}(\theta)\\
        &
\text {s.t.} ~ \pi_\theta=\arg \max _\pi \mathbb{E}_\pi\left[\sum_{t=0}^{\infty} \gamma^t\left(\widehat{r}\left(s_t, a_t ; \theta\right)+\mathcal{H}\left(\pi\left(\cdot \mid s_t\right)\right)\right)\right].
    \end{aligned}
\end{equation}
Note that the problem takes the form of a bi-level optimization problem, where the upper-level problem (ML-IRL) optimizes the reward parameter $\theta$, while the lower-level problem describes the expert's policy as the solution to an entropy-regularized MDP \citep{haarnoja2017reinforcement,haarnoja2018soft}. The entropy regularization in \eqref{empirical formulation}
ensures the uniqueness of the optimal policy $\pi_\theta$ given the fixed reward function $\widehat{r}(s, a; \theta)$ \citep{haarnoja2017reinforcement, haarnoja2018soft}, even when the underlying MDP is high-dimensional and complex.

\section{Proposed Algorithm}

\subsection{Neural Network parameterization of IRL}

\begin{definition}
    The state-action pair $(s, a) \in \mathcal{S} \times \mathcal{A}$ is represented by a vector $x=\psi(s, a) \in \mathcal{X} \subseteq \mathbb{R}^d$ with $d>2$, where $\psi$ is a given one-to-one feature map. 
\end{definition}

With a slight abuse of notation, we use $(s, a)$ and $x$ interchangeably. Without loss of generality, we assume that $\|x\|_2=1$ and $|\widehat{r}(x;\theta)|$ is upper bounded by a constant $\bar{r}>0$ for any $x \in \mathcal{X}$. In this paper, we want to extend our understanding of IRL by replacing the commonly assumed linear reward function with a neural network parameterized one.

To parametrize the reward function and the soft Q value function, we use a two-layer neural network.

\begin{subequations}
\begin{align}
        \label{Q parameterization}
    \widehat{Q}(x ; W)&=\frac{1}{{m}} \sum_{j=1}^m b^j\sigma\left({W^j}^{\top} x\right),\\
    \widehat{r}(x ; \theta)&=\frac{1}{{m}} \sum_{j=1}^m b^j\sigma\left({\theta^j}^{\top} x\right).\label{reward parameterization}
\end{align}
\end{subequations}
Here $\sigma$ is the rectified linear unit (ReLU) activation function. $\sigma(y)=\max \{0, y\}$. $b=\left(b^1, \ldots, b^m\right)$ is generated as $b^j \sim \operatorname{Unif}(\{-1,1\})$ , the initialization parameter $W_0$ is generated as $ W_0^j \sim N\left(0, I_d / d\right)$ , and the initialization parameter $\theta_0$ is initialized as $ \theta_0^j \sim N\left(0, I_d / d\right)$ for any $j \in[m]$ independently. During training, we only update $W=\left(W^1, \ldots,W^m\right)\in \mathbb{R}^{m d}$ and $ \theta=\left(\theta^1, \ldots, \theta^m \right)\in \mathbb{R}^{m d}$, while keeping $b=\left(b^1, \ldots, b^m\right) \in \mathbb{R}^m$ fixed as the random initialization. $W$ is restricted within a closed ball with radius $B$: $S_B=\left\{W \in \mathbb{R}^{m d}:\|W-W_0\|_2 \leq B\right\}(B>0)$.This is a classical neural network structure for theoretical analysis, due to its local linearization property \citep{wang2019neural}. It has been demonstrated to be capable of learning a class of infinite-order smooth functions \citep{arora2019finegrained,allenzhu2020learning}.


To solve problem \eqref{empirical formulation}, \cite{zeng2022maximum} proposes an algorithm named Maximum Likelihood Inverse Reinforcement Learning (ML-IRL) based on certain stochastic algorithms for bi-level optimization \citep{hong2020two,ji2021bilevel,khanduri2021near}, thereby avoiding the heavy computational burden from the traditional nested loop algorithms. Specifically, \textit{soft policy iteration} \citep{haarnoja2017reinforcement} is adopted to optimize the lower-level problem. Therefore, the soft Q-function value under the current step's policy and reward function is critical for policy improvement in ML-IRL.

However, a key limitation of  ML-IRL is that the exact soft Q-function value is typically not accessible. In order to make the algorithm more practical, we decide to incorporate finite-step soft Q-learning \citep{haarnoja2017reinforcement} to directly approximate the optimal soft Q value $\widehat{Q}_{\widehat{r}_{\theta}, \pi_{\theta}}^{s o f t}(s, a)$ under the current reward $\widehat{r}_{\theta}$. The detailed implementation is shown in Algorithm \ref{alg:1}.



\begin{algorithm}[htb]
\caption{Neural Soft Q-learning  }
\label{alg:1}
\begin{algorithmic}[1] 
\REQUIRE ~~ \\ 
   Exploration policy $\pi_{\text{exp}}$ such that $\pi_{\text{exp}}(a|s)>0$ for any $(s,a) \in \mathcal{S} \times \mathcal{A}$, total iterations $T$, stepsize $\eta$,  current reward $\widehat{r}_{\theta}$, $b^j \sim \operatorname{Unif}(\{-1,1\}),$\\
   $ W_0^j \sim N\left(0, I_d / d\right)(s \in[m]), \bar{W}=W_0,
   $\\
   $S_B=\left\{W \in \mathbb{R}^{m d}:\|W-W_0\|_2 \leq B\right\}(B>0) $,\
   \FOR{$t=0,1, \ldots, T-1$}          
            \STATE  Sample a tuple $\left(s, a, \widehat{r}_{\theta}, s^{\prime}\right)$ from the stationary distribution $\mu_{\text{exp}}$ of the exploration policy $\pi_{\text{exp}}$
            \STATE Bellman residual calculation: $\delta \leftarrow \widehat{Q}(s, a ; W_t)-\widehat{r}_{\theta}-\gamma \operatorname{softmax}_{a^{\prime} \in \mathcal{A}} \widehat{Q}\left(s^{\prime}, a^{\prime} ; W_t\right);$
            \STATE TD update: $\widetilde{W}_{t+1} \leftarrow W_t-\eta \delta \cdot \nabla_W \widehat{Q}(s, a ; W_t);$
            \STATE Projection: $W_{t+1} \leftarrow \operatorname{argmin}_{W \in S_B}\|W-\widetilde{W}_{t+1}\|_2$\\
            \STATE Averaging: $\bar{W} \leftarrow \frac{t+1}{t+2} \cdot \bar{W}+\frac{1}{t+2} \cdot W_{t+1}$
        \ENDFOR
            \STATE Output: $ \widehat{Q}_{\text {out }}(\cdot)=\widehat{Q}_{\widehat{r}_{\theta}, \pi_{\theta}}^{s o f t}(\cdot)\leftarrow \widehat{Q}(\cdot ; \bar{W})$
\end{algorithmic}
\end{algorithm}


\subsection{ML-IRL with Dynamically Truncated Soft Q-learning Nested}

Typically, it requires $\mathcal{O}\left(\epsilon^{-2}\right)$ loops of Algorithm \ref{alg:1} to get an $\epsilon-$accurate approximation of the optimal soft Q-function for each policy improvement \citep{cai2023neural}. However, it is not possible to do so in real practice.
To address this, we present an improved and feasible Algorithm \ref{alg:2}. The innovation in Algorithm \ref{alg:2} lies in the dynamic adjustment of the number of stochastic steps in the nested Algorithm \ref{alg:1}. Specifically, at the $k+1$th round of policy improvement, Algorithm \ref{alg:1} is required to perform $k+2$ iterations to approximate the optimal soft Q value under the reward $\widehat{r}_{\theta_k}$ at $k$th iteration.

 This design utilizes soft Q-learning to get a direct but imprecise estimation of optimal soft Q-function value in finite iterations. The previous convergence guarantee in ML-IRL \citep{zeng2022maximum} may no longer hold since Algorithm \ref{alg:2} doesn't achieve an $\epsilon-$accurate policy estimation required for soft policy iteration. However, by forcing the number of nested iterations in Algorithm \ref{alg:1} to increase linearly with outer reward update rounds, we can still ensure global convergence. The convergence guarantee will be discussed in later Section \ref{theoretical}. We now present the Algorithm \ref{alg:2} below.

\begin{algorithm}[htb]
\caption{Maximum Likelihood Inverse Reinforcement Learning (ML-IRL) with Dynamically Truncated Soft Q-learning nested}
\label{alg:2}
\begin{algorithmic}[1] 
\REQUIRE ~~\\ 
    Reward parameter $\theta_0$,
    stepsize of reward parameter $\alpha$;

   \FOR{ $k=0,1, \ldots, K-1$}
   \STATE Optimal Soft Q-function Approximation: Run Algorithm \ref{alg:1} under reward function $\widehat{r}\left(\cdot;\theta_k\right)$ for $T=k+2$ rounds
        \STATE Policy Improvement: \\$\pi_{k+1}(\cdot \mid s) \propto \exp \left(\widehat{Q}_{\widehat{r}_{\theta_k}, \pi_{\theta_k}}^{s o f t}(s, \cdot)\right), \forall s \in \mathcal{S}$.
        \STATE Data Sampling I: Sampling expert trajectory $\tau_k^E:=\left\{s_t, a_t\right\}_{t \geq 0}$ from the dataset $\mathcal{D}$
        \STATE Data Sampling II: Sampling agent trajectory $\tau_k^A:=\left\{s_t, a_t\right\}_{t \geq 0}$ from the policy $\pi_{k+1}$
        \STATE Estimating Gradient: \\$g_k:=h\left(\theta_k ; \tau_k^E\right)-h\left(\theta_k ; \tau_k^A\right)$ where $h(\theta ; \tau):=\sum_{t \geq 0} \gamma^t \nabla_\theta 
 \widehat{r}\left(s_t, a_t ;\theta\right)$
        \STATE Reward Parameter Update: $\theta_{k+1}:=\theta_k+\alpha g_k$
    \ENDFOR
\RETURN $\pi_K,\theta_K$
\end{algorithmic}
\end{algorithm}

\subsection{Two-timescale Single Loop ML-IRL }
As is shown above, Algorithm \ref{alg:2} still retains a nested inner loop structure due to the approximation of the optimal soft Q-function. Despite requiring only a finite number of iterations for each policy improvement round, it still results in an overall complexity of $\mathcal{O}\left(K^{2}\right)$. To reduce the computational complexity, we design a two-timescale single-loop Algorithm \ref{alg:3}. While still relying on the approximate optimal soft Q-function (line 3 to line 8), Algorithm \ref{alg:3} achieves a significant improvement by just one iteration of stochastic update from Algorithm \ref{alg:1}, under the current step's reward. 

To ensure global convergence, we employ the two-timescale stochastic approximation method \citep{hong2020two,borkar1997stochastic} By carefully selecting two distinct stepsizes, the lower-level policy optimization converges more quickly than the upper-level reward updates. This approach enables the policy $\pi_{k+1}$ to stay aligned with the optimal $\pi_{\theta_k}$. In the policy improvement step of Algorithm \ref{alg:3} (line 9), the lower-level policy is iteratively updated by the \textit{soft
policy iteration} \citep{haarnoja2017reinforcement}, which converges linearly to the optimal policy under a fixed reward function \citep{cen2022fast}. As a result, the lower-level updates occur on a faster timescale. In contrast, the upper-level reward parameter updates (line 13) progress more slowly because it does not exhibit such a nice linear convergence
property. A more detailed discussion on choosing the appropriate stepsizes can be found in Supplementary \ref{A3_policy_proof}. We give the full implementation details of Algorithm \ref{alg:3} as follows.
\begin{algorithm}[htb]
\caption{Two-Timescale Single Loop Maximum Likelihood Inverse Reinforcement Learning (ML-IRL)}
\label{alg:3}
\begin{algorithmic}[1] 
\REQUIRE ~~\\ 
    Reward parameter $\theta_0$,
    explorative policy $\pi_0$,
    stepsize of reward parameter update $\alpha$, stepsize of TD update $\eta$.
 \STATE initialize $b^j \sim \operatorname{Unif}(\{-1,1\}),$ \\$ \quad W_0^j \sim N\left(0, I_d / d\right)(s \in[m]), \bar{W}=W_0,$\\
 $\quad S_B=\left\{W \in \mathbb{R}^{m d}:\|W-W_0\|_2 \leq B\right\}(B>0) $;\
   \FOR{ $k=0,1, \ldots, K-1$} 
            \STATE  Sample a tuple $\left(s, a, r, s^{\prime}\right)$ from the stationary distribution $\mu_{k}$ of the current policy $\pi_{k}$
            \STATE Bellman residual calculation: $\delta \leftarrow \widehat{Q}(s, a ; W_k)-\widehat{r}\left(s,a ;\theta_k\right)-\gamma \operatorname{softmax}_{a^{\prime} \in \mathcal{A}} \widehat{Q}\left(s^{\prime}, a^{\prime} ; W_k\right);$
            \STATE TD update: $\widetilde{W}_{k+1}\leftarrow W_k-\eta \delta \cdot \nabla_W \widehat{Q}(s, a ; W_k);$
            \STATE Projection: $W_{k+1} \leftarrow \operatorname{argmin}_{W \in S_B}\|W-\widetilde{W}_{k+1}\|_2$\\
            \STATE Averaging: $\bar{W} \leftarrow \frac{k+1}{k+2} \cdot \bar{W}+\frac{1}{k+2} \cdot W_{k+1}$
       
            \STATE Output: $ \widehat{Q}_{\widehat{r}_{\theta_k}, \pi_{\theta_k}}^{s o f t}( \cdot)=\widehat{Q}_{\text {out }}(\cdot) \leftarrow \widehat{Q}(\cdot ; \bar{W})$
        \STATE Policy Improvement: \\$\pi_{k+1}(\cdot \mid s) \propto \exp \left(\widehat{Q}_{\widehat{r}_{\theta_k}, \pi_{\theta_k}}^{s o f t}(s, \cdot)\right), \forall s \in \mathcal{S}$.
        \STATE Data Sampling I: Sampling expert trajectory $\tau_k^E:=\left\{s_t, a_t\right\}_{t \geq 0}$ from the dataset $D$
        \STATE Data Sampling II: Sampling agent trajectory $\tau_k^A:=\left\{s_t, a_t\right\}_{t \geq 0}$ from the policy $\pi_{k+1}$
        \STATE Estimating Gradient:\\ $g_k:=h\left(\theta_k ; \tau_k^E\right)-h\left(\theta_k ; \tau_k^A\right)$, where $h(\theta ; \tau):=\sum_{t \geq 0} \gamma^t \nabla_\theta \widehat{r}\left(s_t, a_t ; \theta\right)$
        \STATE Reward Parameter Update: $\theta_{k+1}:=\theta_k+\alpha g_k$
    \ENDFOR
    
\RETURN $\pi_K,\theta_K$
\end{algorithmic}
\end{algorithm}

\section{Theoretical Analysis}
\label{theoretical}
In this section, we mainly provide theoretical insights into our proposed algorithms.

\subsection{Definition:}
Recall that the soft Q function has the following overparameterization structure:
\begin{equation*}
    \widehat{Q}(x ; W) =\frac{1}{{m}} \sum_{j=1}^m b^j \mathbbm{1}\left\{{W^j}^{\top} x>0\right\} {W^j}^{\top}x,
\end{equation*}
and $\nabla_{W^j} \widehat{Q}(x ; W)=\frac{b^j}{m} \mathbbm{1}\left\{{W^j}^{\top} x>0\right\} x$ almost everywhere in $\mathbb{R}^{ d}$.

\begin{definition}
    (Local Linearization of soft Q function) The neural network $\widehat{Q}(x ; W)$  can be linearized locally at random initialization point $W_0$ by $\widehat{Q}_0(x ; W)$ with respect to $W$:
    \begin{equation}
    \label{Q0_def}
        \widehat{Q}_0(x ; W)=\Phi(x)^{\top} W,
    \end{equation}
    where we define a feature map $\Phi(x) \in \mathbb{R}^{m d}$ as follows: $ \Phi(x):=\frac{1}{{m}} \cdot\left(\mathbbm{1}\left\{{W_0^1}^{\top} x>0\right\} x, \ldots, \mathbbm{1}\left\{{W_0^m}^{\top} x>0\right\} x\right).$
\end{definition}
$\widehat{Q}_0(x ; W)$ is linear in the feature map $\Phi(x)$.



\textbf{Remark:} Local linearization of the neural network is a classic function class of nonlinear function approximation and is known to have strong representation power when the width m goes to infinity \citep{hofmann2008kernel}. It is commonly used in the literature of stationary point convergence analysis \citep{zou2018stochastic, cai2023neural,allenzhu2020learning,daniely2017sgdlearnsconjugatekernel}.


\subsection{Assumptions}
\begin{assumption}
    \label{Ergodicity}
    (Ergodicity) For any policy $\pi$, assume the Markov chain with transition kernel $\mathcal{P}$ is irreducible and aperiodic under policy $\pi$. Then there exist constants $\kappa>0$ and $\rho \in(0,1)$ such that
$$
\sup _{s \in \mathcal{S}}\left\|\mathbb{P}\left(s_t \in \cdot \mid s_0=s, \pi\right)-\mu_\pi(\cdot)\right\|_{T V} \leq \kappa \rho^t, \quad \forall t \geq 0,
$$
where $\|\cdot\|_{T V}$ is the total variation (TV) norm. $\mu_\pi$ is the stationary state distribution under $\pi$.

\end{assumption}

Assumption \ref{Ergodicity} assumes the Markov chain mixes at a geometric rate. It is a common assumption in the literature of RL \citep{wu2022finite,bhandari2018finite,zou2019finitesample}, which holds for any time-homogeneous Markov chain with finite-state space or any uniformly ergodic Markov chain with general state space.



\begin{assumption}
    \label{regularity of policy}
 (Regularity of Policy)  There exists a constant $\nu^{\prime}>0$ such that for any $W_1, W_2 \in S_B$, it holds that
{
\begin{equation*}
\left(\gamma+\nu^{\prime}\right)^{-2} \mathbb{E}_{x \sim \mu}\left[\left(\widehat{Q}_0\left(x ; W_1\right)-\widehat{Q}_0\left(x ; W_2\right)\right)^2\right]\geq \mathbb{E}_{s \sim \mu}\Bigg[\left(\underset{a \in\mathcal{A}}{\operatorname{softmax}}\widehat{Q}_0\left(s, a ; W_1\right)-\underset{a \in \mathcal{A}}{\operatorname{softmax}} \widehat{Q}_0\left(s, a ; W_2\right)\right)^2\Bigg].
\end{equation*}
}

\end{assumption}

\cite{cai2023neural} proposes this strong regularity assumption to show that the soft Bellman operator is a contraction mapping and therefore guarantees the uniqueness of our approximate stationary point $W_{\theta}^*$.
\cite{zou2019finitesample,melo2008analysis}'s discussions on the global convergence of linear Q-learning rely on a more strong assumption that implies Assumption \ref{regularity of policy}.

\begin{assumption}
    \label{Regularity of Stationary Distribution} (Regularity of Stationary Distribution) There exists a constant $c_1>0$ such that for any $\zeta \geq 0$ and $w \in \mathbb{R}^d$ with $\|w\|_2=1$, it holds that
$$
\mathbb{P}\left(\left|w^{\top} \psi(s, a)\right| \leq \zeta, \text { for all } a \in \mathcal{A}\right) \leq c_1 \cdot \zeta,
$$
where $(s, a) \sim \mu$.
\end{assumption}

Assumption \ref{Regularity of Stationary Distribution} imposes a constraint on the density of $\mu$ with respect to the marginal distribution of $x$. This constraint is naturally satisfied when the marginal distribution of $x$ has a uniformly bounded probability density across the unit sphere, as discussed in \cite{cai2023neural}.

\subsection{Theorem}


\begin{theorem}
\label{Theorem 1}
    Suppose Assumptions \ref{Ergodicity}, \ref{regularity of policy} and \ref{Regularity of Stationary Distribution} hold. Selecting stepsize $\alpha:=\frac{\alpha_0}{K^\sigma}$ for the reward update step and $\eta=\min \{K^{-\frac{1}{2}},(1-\gamma) / 8\}$ for the TD update step in Algorithm \ref{alg:1} where $\alpha_0>0$ and $\sigma \in(0,1)$ are some fixed constants, and $K$ is the total number of iterations of the Algorithm \ref{alg:2}. Then the following result holds:

    \begin{subequations}
    \begin{align}
        & \frac{1}{K} \sum_{k=0}^{K-1} \mathbb{E}\left[\left\|\log \pi_{k+1}-\log \pi_{\theta_k}\right\|_{\infty}\right] =\mathcal{O}\left(K^{-\frac{1}{4}}\right)
            +\mathcal{O}\left(B^{3/2} m^{-1 / 2}+B^{5 / 4} m^{-1 / 4}\right), \\
            &\frac{1}{K} \sum_{k=0}^{K-1} \mathbb{E}\left[\left\|\nabla \widehat{\mathcal{L}}\left(\theta_k\right)\right\|^2\right] =\mathcal{O}\left(K^{-\sigma}\right)+\mathcal{O}\left(K^{-1+\sigma}\right)+\mathcal{O}\left(K^{-\frac{1}{4}}\right) +\mathcal{O}\left(m^{-1/2}+B^{3/2} m^{-1 / 2}+B^{5 / 4} m^{-1 / 4}\right),
    \end{align}
     \end{subequations}
where the expectation is over all randomness and \\$\left\|\log \pi_{k+1}-\log \pi_{\theta_k}\right\|_{\infty}$$=\max _{s \in \mathcal{S}, a \in \mathcal{A}}\left|\log \pi_{k+1}(a \mid s)-\log \pi_{\theta_k}(a \mid s)\right|$.
\end{theorem}

In Theorem \ref{Theorem 1}, we present the non-asymptotic convergence guarantee for the Algorithm \ref{alg:2}. In particular, by setting $\sigma=-\frac{1}{4}$ and the width $m$ to infinity, it takes both reward and policy $\mathcal{O}(\epsilon^{-4})$ iterations to converge to an $\epsilon-$stationary point. The detailed proof of Theorem \ref{Theorem 1} is in Supplementary \ref{thm 1 proof}.

\begin{theorem}
\label{Theorem 2}
    Suppose Assumptions \ref{Ergodicity}, \ref{regularity of policy} and \ref{Regularity of Stationary Distribution} hold. Selecting stepsize $\alpha:=\frac{\alpha_0}{K^\sigma}$ for the reward update step and  $\eta=\min \{K^{-\frac{3}{4}},(1-\gamma) / 8\}$ for the TD update step where $\alpha_0>0$ and $\sigma \in(0,1)$ are some fixed constants, and $K$ is the total number of iterations to be run by the Algorithm \ref{alg:3}. Then the following result holds:
    \begin{subequations}
    \begin{align}
        &\quad \frac{1}{K} \sum_{k=0}^{K-1} \mathbb{E}\left[\left\|\log \pi_{k+1}-\log \pi_{\theta_k}\right\|_{\infty}\right]=\mathcal{O}\left(K^{-\frac{1}{8}}\right)
            +\mathcal{O}\left(B^{3/2} m^{-1 / 2}+B^{5 / 4} m^{-1 / 4}\right), \label{thm2.a}\\
            &\quad\frac{1}{K} \sum_{k=0}^{K-1} \mathbb{E}\left[\left\|\nabla \widehat{\mathcal{L}}\left(\theta_k\right)\right\|^2\right] =\mathcal{O}\left(K^{-\sigma}\right)+\mathcal{O}\left(K^{-1+\sigma}\right)+\mathcal{O}\left(K^{-\frac{1}{8}}\right) +\mathcal{O}\left(m^{-1/2}+B^{3/2} m^{-1 / 2}+B^{5 / 4} m^{-1 / 4}\right), \label{thm2.b}
    \end{align}
 \end{subequations}
where the expectation is over all randomness and \\$\quad\left\|\log \pi_{k+1}-\log \pi_{\theta_k}\right\|_{\infty}$
$=\max _{s \in \mathcal{S}, a \in \mathcal{A}}\left|\log \pi_{k+1}(a \mid s)-\log \pi_{\theta_k}(a \mid s)\right|$.
\end{theorem}
Theorem \ref{Theorem 2} shows that our two-timescale single-loop Algorithm \ref{alg:3} can identify an approximate stationary point of our problem \eqref{empirical formulation} up to a neural network error, which depends on the width of our neural network. Specifically, when $\sigma=-\frac{1}{8}$ and the neural network is sufficiently wide, both the reward and policy converge with the rate of  $\mathcal{O}(K^{-\frac{1}{8}})$ iterations to an $\epsilon-$stationary point.

\textbf{Remark:} Our theoretical guarantee stands apart from the existing results in several key ways. \cite{cen2022fast} only shows the convergence rate of soft policy iteration under a fixed reward function. \cite{zeng2022maximum} first considers the setting where both the policy and the reward parameter keep changing, but the convergence is only guaranteed when the reward function is linear. In contrast, we analyze a significantly more challenging scenario where the reward is parameterized by neural networks, providing a broader and more generalized convergence guarantee. We present our proof in Supplementary \ref{A3_policy_proof}.


\textbf{Discussions:} \cite{allenzhu2019convergencetheorydeeplearning,gao2019convergence,zou2018stochastic} go beyond the double-layer neural networks and establish high probability convergence guarantees for deep overparameterized neural networks. \cite{cai2023neural} then apply these results to extend the analysis of neural soft Q-learning. Therefore, our Theorem \ref{Theorem 1} and \ref{Theorem 2} can be generalized to the setting where both the reward function and soft Q-function are parameterized by deep neural networks with slight modifications to network errors. We just refer to these findings here to support our experimental setup in a later section,  since it is not the main focus of this work. 

Next, we demonstrate that our Algorithm \ref{alg:3} can converge to the globally optimal solution of problem \eqref{empirical formulation} under overparameterization.
To study the global optimality of our algorithm, we reformulate the bi-level problem as the following equivalent saddle point problem \citep{zeng2024demonstrationsmeetgenerativeworld}. 
{
\begin{equation}
\label{new_formulation}
   \max _\theta \min _\pi 
 \quad \widehat{L}(\theta, \pi):=\mathbb{E}_{\tau^{E} \sim \mathcal{D}}\left[\sum_{t=0}^{\infty} \gamma^t\left(\widehat{r}\left(s_t, a_t ; \theta\right))\right)\right]-\mathbb{E}_{\tau^{\mathrm{A}} \sim(\mu_0, \pi)}\left[\sum_{t=0}^{\infty} \gamma^t\left(\widehat{r}\left(s_t, a_t ; \theta\right)+\mathcal{H}\left(\pi\left(\cdot \mid s_t\right)\right)\right)\right] .
\end{equation}
}

Since $\pi_{\theta}$ is defined as the optimal policy under reward parameter $\theta$ in the bi-level formulation \eqref{empirical formulation}, the following relation holds:
\begin{equation}
\label{bi-minmax}
    \widehat{\mathcal{L}}(\theta)=\min _\pi \widehat{L}(\theta, \pi).
\end{equation}

\begin{theorem}
\label{Thm: concavity}
     (Approximate Concavity)
    Given any fixed policy $\pi$ and for any $\theta, \theta^{\prime} \in \mathbb{R}^{m d}$ we have:
    \begin{equation*}
 \left[\widehat{L}\left( \theta^{\prime},\pi\right)-\widehat{L}\left( \theta,\pi\right)\right]\leq \left[\nabla_\theta \widehat{L}\left( \theta,\pi\right)^{\top}\left(\theta^{\prime}-\theta\right)\right]+\mathcal{O}\left(m^{-1 / 2}\right).
    \end{equation*}
\end{theorem}

Therefore, we note that the max-min objective $\widehat{L}(\theta, \pi)$ is concave in the reward parameter $\theta$ as the width $m$ of our neural network tends to infinity. 
Besides, we establish the following relationship:$\nabla \widehat{\mathcal{L}}({\theta})= \nabla_\theta \widehat{L}(\theta, \pi_{\theta})$ under overparameterization. We can utilize this relation and the concavity in Theorem \ref{Thm: concavity} to show that any stationary point $\tilde{\theta}$ of \eqref{empirical formulation} together with its corresponding optimal policy $\pi_{\tilde{\theta}}$ consists of a saddle point of \eqref{new_formulation} by checking the first-order condition. From \eqref{bi-minmax}, we derive the following result for a saddle point $\left(\tilde{\theta}, \pi_{\tilde{\theta}}\right)$:

 \[\tilde{\theta} \in \arg  \max _\theta \min _{\pi_{\theta}}\widehat{L}(\theta, \pi)=\arg \max _\theta \widehat{\mathcal{L}}(\theta) .\]

Therefore, for any saddle point $\left(\tilde{\theta}, \pi_{\tilde{\theta}}\right)$ of the objective $\widehat{L}(\cdot,\cdot)$, the reward parameter $\tilde{\theta}$ is a global optimal solution of  problem \eqref{empirical formulation}. The detailed proof can be found in Supplementary \ref{Optimality_supp_proof}. 

\begin{corollary}
    Assume that the ground truth reward value $r\left(s, {a}; \theta^*\right)$ for an action ${a} \in A$ and $s \in S$ is known and that the set of parameters $ \Theta$ is a compact set. Let $\widetilde{\theta}$ be the optimal solution to \eqref{empirical formulation}. If the number of expert trajectories in the demonstration data set satisfies $|\mathcal{D}| \geq \frac{2}{\epsilon^2m^2(1-\gamma)^2} \ln \left(\frac{2}{\delta}\right)$, then with probability greater than $1-\delta$, we obtain the performance guarantee of the optimality gap in \eqref{formulation} that:
    \begin{equation*}
{\mathcal{L}}\left(\theta^*\right)-{\mathcal{L}}(\tilde{\theta}) \leq \epsilon.
    \end{equation*}
\end{corollary}

Theorem \ref{Theorem 1} and \ref{Theorem 2} have already shown that our learned reward converges to the staionary point of the empirical estimation \eqref{empirical formulation}, i.e. ${\theta}_K \rightarrow \tilde{\theta}$. Since the parameter set $\Theta$ is compact, therefore the set of limit points of the sequence $\left\{{\theta}_K: K \in \mathbb{N}^{+}\right\}$ is non-empty. The uniqueness of the global optimal point ${\theta^*}$ is established by the approximate concavity in Theorem \ref{Thm: concavity} with the auxiliary Assumption \ref{regularity of policy}. Therefore, the set of limit points is a singleton.  To prove whether the global optimal reward from the empirical estimation problem \eqref{empirical formulation} is close to the underlying true reward, we investigate the optimality gap $\mathcal{L}\left(\theta^*\right)-\mathcal{L}(\tilde{\theta})$. We give the proof sketch as follows. 

\textbf{Proof Skecth:}
With probability greater than $1-\delta$, the following relation holds:
\begin{align*}
\mathcal{L}\left(\theta^*\right)-\mathcal{L}(\tilde{\theta})&= \left(\mathcal{L}\left(\theta^*\right)-\widehat{\mathcal{L}}\left(\theta^* ; \mathcal{D}\right)\right)+\left(\widehat{\mathcal{L}}\left(\theta^* ; \mathcal{D}\right)-\widehat{\mathcal{L}}(\tilde{\theta} ; \mathcal{D})\right)+(\widehat{\mathcal{L}}(\tilde{\theta} ; \mathcal{D})-\mathcal{L}(\tilde{\theta}))\\
& \stackrel{(i)}{\leq} \frac{1}{m(1-\gamma)} \sqrt{\frac{\ln (2 / \delta)}{2|\mathcal{D}|}}+\left(\widehat{\mathcal{L}}\left(\theta^*\right)-\widehat{\mathcal{L}}(\tilde{\theta})\right)+\frac{1}{m(1-\gamma)} \sqrt{\frac{\ln (2 / \delta)}{2|\mathcal{D}|}},
\end{align*}
where $(i)$ comes from hoefffing inequality.  

Since $\tilde{\theta}$ is defined as the optimal solution to problem \eqref{empirical formulation}, we know that $\widehat{\mathcal{L}}(\theta ; \mathcal{D})-\widehat{\mathcal{L}}(\tilde{\theta} ; \mathcal{D}) \leq 0$ for any $\theta$. Therefore, we yield that, with probability at least $1-\delta$, 
\begin{equation*}
    \mathcal{L}\left(\theta^*\right)-\mathcal{L}(\tilde{\theta}) \leq \frac{2}{m(1-\gamma)} \sqrt{\frac{\ln (2 / \delta)}{2|\mathcal{D}|}} .
\end{equation*}

When the number of expert trajectories in the demonstration data set satisfies $|\mathcal{D}| \geq \frac{2}{\epsilon^2m^2(1-\gamma)^2} \ln \left(\frac{2}{\delta}\right)$, then with probability greater than $1-\delta$, we obtain the performance guarantee of the optimality gap that 
$$
\mathcal{L}\left(\theta^*\right)-\mathcal{L}(\tilde{\theta}) \leq \epsilon.
$$

So far, we have demonstrated that our algorithm can theoretically recover the ground truth reward estimator under overparameterization.

\section{Experiments}
In this section, we mainly compare our single-loop ML-IRL with two classes of algorithms: (a) IRL algorithms that learn the reward function and its corresponding policy simultaneously, such as $f$-IRL \citep{ni2021f} and IQ-Learn \citep{garg2021iq}. (b) Imitation learning algorithms that only learn the policy to imitate the expert, including Behavior Cloning (BC) \citep{pomerleau1988alvinn} and Generative Adversarial Imitation Learning (GAIL) \citep{ho2016generative}. 

\textbf{Experimental Setup:} We test the performance of different algorithms on several high-dimensional robotics control tasks in Mujoco \citep{todorov2012mujoco}. To ensure a fair comparison with benchmark algorithms under these settings, we choose soft 
Actor-Critic \citep{haarnoja2018soft} as the base RL algorithm due to its strong performance in environments with continuous action spaces. Besides, soft Actor-Critic is a special variant of the policy gradient algorithm and equivalent to soft Q-learning \citep{haarnoja2018soft, schulman2018equivalencepolicygradientssoft}. \cite{cai2023neural} extends the convergence analysis of neural soft Q-learning to soft Actor-Critic. Therefore, it is reasonable for us to use soft Actor-Critic to update our policy in the numerical implementation. 

\floatplacement{table*}{t!}
\begin{table*}[t!]
    \centering
    \begin{tabular}{c|c|c|c|c|c|c}
    \hline
    Task&IQ-Learn&BC&GAIL&$f$-IRL&Single Loop ML-IRL&Expert\\
    \hline
    Hopper &\textbf{3546.4*} & 20.53 ± 4.16& 2757.88 ± 293.18&2993.54 ± 252.59& 3262.89 ± 24.31&3592.63\\
    Half-Cheetah &  5043.26&-1.77 ±0.27&3254.51 ± 87.40&4650.49 ± 180.94 & \bf{5074.51 ± 249.70} & 5098.30\\
    Walker &5134.0*&-13.98 ± 0.22 & 1023.41 ± 81.64 & 4483.14 ± 155.72& \bf{5182.09 ± 74.69} &5344.21 \\
 
    Ant &4362.9*& 760.14 ± 0.53 & 1107.98 ± 531.90& 4853.53 ± 213.72 & \bf{4919.80 ± 11.52} & 5926.18\\
    \hline
    \end{tabular}
   \caption{\footnotesize{\bf Mujoco Results.} The performance of benchmark algorithms under a single expert trajectory.}

\end{table*}

In SAC for the baseline algorithms, both policy network and Q-network are $(64,64)$ MLPs with ReLU activation function, and stepsizes as $3 \times 10^{-3}$. We use their open-source implementations in our experiments for all mentioned imitation learning / IRL benchmark algorithms\footnote{\url{https://github.com/Div99/IQ-Learn}}\footnote{\url{https://github.com/KamyarGh/rl_swiss}}. In our proposed single-loop algorithm\footnote{\url{https://github.com/Cloud0723/ML-IRL}}, we parameterize the reward function by a $(64,64)$ MLPs with ReLU activation function. For the reward network, we use Adam as the optimizer, and the stepsize is set to be $1 \times 10^{-4}$. At the beginning of each iteration, we warm-start the SAC algorithm by initializing both the policy and Q-networks with the neural networks trained in the previous iteration. We then run 10 episodes in the corresponding Mujoco environment to further train these two networks. Following this, both agent and expert trajectories are collected to update the reward network with a single gradient step.

\textbf{Numerical Results: }In each experiment, a limited data regime is considered where the expert dataset only contains a single expert trajectory. We use the expert data provided in the open-source implementation of f-IRL \footnote{\url{ https://github.com/twni2016/f-IRL}}.
We train the algorithms until they converge. The accumulated reward scores reported in Table 1 are averaged over 3 random seeds. In each random seed, we train the algorithms from initialization and collect 20 trajectories to average their accumulated rewards after the algorithms converge. The results presented in Table 1 indicate that our proposed single loop algorithm have better performances than existing baseline methods on the majority of tasks.

As we can observe, BC has the poorest ability to imitate our expert's behavior. This aligns with the fact that BC relies on supervised learning, which fails to learn the policy well within a limited data regime. Furthermore, we note that the training process of IQ-Learn is unstable, potentially due to its imprecise approximation of the soft Q-function. As a result, in the Half-Cheetah task, we are unable to replicate the results reported in the original paper \citep{garg2021iq} and have to directly report the results from the original paper, which is indicated by an asterisk (*) in Table 1. The results of AIRL are excluded from Table 1 due to its consistent poor performance, despite our extensive parameter tuning efforts. \citep{ni2021f,liu2019statealignmentbasedimitationlearning} have encountered similar issues.

\section{Conclusions}

In this paper, we advance the understanding of the maximum likelihood IRL framework and propose two provably efficient algorithms, which can recover the neural network parameterized reward function and the policy simultaneously. To our knowledge, Algorithm \ref{alg:3} is the first single-loop algorithm to address the neural network parameterized IRL problem with solid theoretical guarantees. Besides, extensive numerical experiments on Mujoco tasks demonstrate the superiority of our single-loop ML-IRL algorithm over traditional methods. A limitation of our method is the requirement of a sufficiently large width of the parameterization neural network, so one future direction of this work is to further extend our theoretical analysis to the finite-width case.

\section*{Acknowledgements}
M. Hong and S. Zeng are supported partially by NSF under the grants EPCN-2311007, ECCS-2426064 and CCF-2414372, also by Minnesota Supercomputing Institute. A. Garcia and C. Li are partially supported by W911NF-22-1-0213, ECCS-2240789 and CCF-2414373.

\bibliography{References}

\input{supplement}

\end{document}

%% file: supplement.tex
\onecolumn

\setcounter{section}{0}
\renewcommand{\thesection}{\Alph{section}}
\setcounter{definition}{0}
\renewcommand{\thedefinition}{\Alph{definition}}

\section{\texorpdfstring{Function Approximation of Soft Q-learning \cite{cai2023neural}}{Function Approximation of Soft Q-learning}}
\begin{definition}
    The soft Bellman optimality operator is defined as: 
\begin{equation}
  \label{eq1}
  \mathcal{T}(Q^{\rm soft}_{\widehat{r}_{\theta_k},\pi})(s, a)=\mathbb{E}\left[r(s, a)+
\gamma \cdot \underset{a^{\prime} \in \mathcal{A}}{\operatorname{softmax}} Q^{\rm soft}_{\widehat{r}_{\theta_k},\pi}\left(s^{\prime}, a^{\prime}\right) \mid s^{\prime} \sim \mathcal{P}(\cdot \mid s, a)\right],
\end{equation}

\end{definition}
for which $Q^{\rm soft}_{\widehat{r}_{\theta_k},\pi}$ is the fixed point. In other words, $Q^{\rm soft}_{\widehat{r}_{\theta_k},\pi}=\mathcal{T}(Q^{\rm soft}_{\widehat{r}_{\theta_k},\pi}),$
 which minimizes the following Mean Squared Bellman Error (MSBE) problem: $\mathbb{E}_{(s,a,s^\prime)\sim \mu}\left[\left(Q(x)-\mathcal{T} Q(x)\right)^2\right].$

In this work, we utilize the neural network as a nonlinear function approximation. An approximate function class $\mathcal{F}_{W}$ is defined as $\left\{\widehat{Q}(\cdot;W)+\nabla_{W} \widehat{Q}(\cdot;W)^{\top}\left(W^\prime-W\right): W^\prime \in S_B\right\}$ 
. $S_B=$ $\left\{W^\prime \in \mathbb{R}^{m\times d}:\|W^\prime-W\|_2 \leq B\right\}$ is the parameter space of feasible $W$, which consists of the local linearization of $\widehat{Q}(\cdot;W^\prime)$ at point $W$ \citep{cai2023neural}. The bounded domain  limits the distance between the parameters $W^\prime$ and the initial parameters 
$W$, enhancing training stability. Since we know the distribution of initialization $W_0$, we take $W=W_0$ here. Recall that we adopt a double-layer neural network with ReLU activation, the approximate function class can be defined as follows.
\begin{definition}
    (Approximate Function Class)
    \label{Approximate Function Class}
    $$\mathcal{F}_{W_0}=\left\{\frac{1}{\sqrt{m}} \sum_{r=1}^m b^j \mathbbm{1}\left\{{W_0^j}^{\top} x>0\right\} {W^j}^{\top} x: W \in S_B\right\}.$$
\end{definition}

To find the best neural network approximation of $Q^{\rm soft}_{\widehat{r}_{\theta_k},\pi}$, we usually turn to seek the stationary point $\widehat{Q}_0(\cdot;W_{\theta_k}^*)$ of a surrogate measure for the MSBE on $\mathcal{F}_{W_0}$, which is referred to as the Mean Squared Projected Bellman Error (MSPBE).
\begin{equation}
\label{MSPBE}
    \min _W \operatorname{MSPBE}(W)=\mathbb{E}_{(s, a) \sim \mu}\left[\left(\widehat{Q}(s, a;W)-\Pi_{\mathcal{F}_{W}} \mathcal{T} \widehat{Q}(s, a;W)\right)^2\right],
\end{equation}

where the soft Q-function is parametrized by $\widehat{Q}(s, a;W)$ with parameter $W$. In this context, $\mu$ signifies the stationary distribution of $(s, a)$ induced by the policy $\pi$. The projection onto a function class $\mathcal{F}_{W}$ is represented as $\Pi_{\mathcal{F}_{W}}$. We then introduce two lemmas to support the feasibility of focusing on the stationary point of MSPBE for convergence analysis.

\begin{lemma}
    (Lemma 4.2 in \cite{cai2023neural})  \big (Existence, Uniqueness, and Optimality of $\widehat{Q}_0\left(\cdot; W_{\theta_k}^*\right)$\big). There exists a stationary point $W_{\theta_k}^*$ for any $b \in \mathbb{R}^m$ and $W_0 \in \mathbb{R}^{m\times d}$. Also, $\widehat{Q}_0\left(\cdot; W_{\theta_k}^*\right)$ is unique almost everywhere and is the global optimum of the MSPBE that corresponds to the projection onto $\mathcal{F}_{W_0}$.
\end{lemma} 

This lemma establishes the unique existence of an approximate stationary point, as it corresponds to the fixed point of the operator $\Pi_{\mathcal{F}_{W_0}} \mathcal{T}$. This is because the projection operator is an $\ell_2$-norm contraction, associated with the stationary distribution $\mu.$


\begin{lemma} (Proposition 4.7 in \cite{cai2023neural})
\label{Function gap}
\begin{equation}
    \left\|\widehat{Q}_0\left(\cdot ; W_{\theta_k}^*\right)-Q^{\rm soft}_{\widehat{r}_{\theta_k},\pi}(\cdot)\right\|_\mu \leq(1-\gamma)^{-1} \cdot\left\|\Pi_{\mathcal{F}_{B, m}} Q^{\rm soft}_{\widehat{r}_{\theta_k},\pi}-Q^{\rm soft}_{\widehat{r}_{\theta_k},\pi}(\cdot)\right\|_\mu,
\end{equation}
where $Q^{\rm soft}_{\widehat{r}_{\theta_k},\pi}$ is the fixed point of soft Bellman optimality operator defined in \eqref{eq1}.
\end{lemma}

When the neural network is overparameterized with width $m \rightarrow \infty$ and sufficiently large $S_B$, the projection $\Pi_{\mathcal{F}_{B, \infty}}$ reduces to the identity operator \cite{hofmann2008kernel}. By Lemma \ref{Function gap}, $\widehat{Q}_0\left(\cdot; W_{\theta_k}^*\right)=Q^{\rm soft}_{\widehat{r}_{\theta_k},\pi}(\cdot)$. Therefore, $\widehat{Q}_0\left(\cdot; W_{\theta_k}^*\right)$ can be seen as the global optimum of the MSBE. 


Therefore, it suffices to consider $\mathcal{F}_{W_0}$ in place of $\mathcal{F}_{W^\dagger}$ for our convergence analysis.

\section{Auxiliary Lemmas}
Assumptions \ref{Ergodicity}, \ref{regularity of policy} and \ref{Regularity of Stationary Distribution} hold throughout the entire section. Necessary proofs of some auxiliary lemmas can be found in the next Section \ref{Lemma_proof}.

Due to the fact that the optimal policy has the closed form $\pi_\theta(\cdot \mid s) \propto \exp \left(Q_{r_\theta, \pi_\theta}^{\text {soft }}(s, \cdot)\right)
$ \citep{haarnoja2017reinforcement}, we can establish the following lemma to express the objective function $\widehat{\mathcal{L}}(\theta)$ in an equivalent form. 

\begin{lemma}
\label{objective reformulation}
    \begin{equation}
        \widehat{\mathcal{L}}\left(\theta \right)=\mathbb{E}_{\tau^E \sim D}\left[\sum_{t=0}^{\infty} \gamma^t \widehat{r}\left(s_t, a_t ; \theta\right)\right]-\mathbb{E}_{\tau^{\mathrm{A}} \sim \pi_\theta}\left[\sum_{t=0}^{\infty}\gamma^t \left(\widehat{r}\left(\mathbf{s}_t, \mathbf{a}_t; \theta \right)+\mathcal{H}\left(\pi\left(\cdot \mid \mathbf{s}_t\right)\right)\right)\right] .
    \end{equation}
\end{lemma}

The equivalent formulation of the objective directly leads to the following lemma without proof, providing a closed-form expression for the gradient of the objective function.

\begin{lemma} (Lemma 4.1 in \cite{zeng2022maximum})
\label{L gradient}
    The gradient of the likelihood objective $\widehat{\mathcal{L}}(\theta)$ in \eqref{L_theta} can be equivalently expressed as follows:
    \begin{equation}
        \nabla \widehat{\mathcal{L}}(\theta)= \mathbb{E}_{\tau^E \sim \mathcal{D}}\left[\sum_{t \geq 0} \gamma^t \nabla_\theta \widehat{r}\left(s_t, a_t ; \theta\right)\right]-\mathbb{E}_{\tau^{\mathrm{A}} \sim \pi_\theta}\left[\sum_{t \geq 0} \gamma^t \nabla_\theta \widehat{r}\left(s_t, a_t ; \theta\right)\right] .
    \end{equation}

\end{lemma}

Before we further investigate the theoretical properties of likelihood objective $\widehat{\mathcal{L}}(\theta)$, we continue introducing some related lemmas:

\begin{lemma}
    \label{Lipschitz reward Lemma} (Proposition 1 in \cite{scaman2019lipschitz}, Lemma 5.3 in \cite{zeng2022maximum} )
     Let soft Q-function and reward function be defined as \eqref{Q parameterization} \eqref{reward parameterization} respectively. For any $s \in \mathcal{S}, a \in \mathcal{A}$ and any reward parameter $\theta$, the following holds:    
\begin{equation}  
\label{r lip}
    \left|\widehat{r}(s, a;\theta_1)-\widehat{r}(s, a;\theta_2)\right| \leq \frac{1}{\sqrt{m}} \left\|\theta_1-\theta_2\right\|,
\end{equation}

\begin{equation}
\label{q lip}
    \left|\widehat{Q}(s, a;W_1)-\widehat{Q}_(s, a;W_2)\right| \leq \frac{1}{\sqrt{m}}\left\|W_1-W_2\right\|,
\end{equation}

\begin{equation}
\label{soft q lip}
    \left|Q_{r_{\theta_1}, \pi_{\theta_1}}^{\text {soft}}(s, a)-Q_{r_{\theta_2}, \pi_{\theta_2}}^{\text {soft}}(s, a)\right| \leq \frac{1}{\sqrt{m}(1-\gamma)}\left\|\theta_1-\theta_2\right\|.
\end{equation}
\end{lemma}
This lemma shows that the parameterized reward and soft Q-function are Lipschitz continuous, and therefore have a bounded gradient. We note that it's a natural result of our specific neural network structure, which relaxes the commonly used Lispchitz continuity assumption in the literature.

\begin{lemma}
\label{lemma B.1}
(\cite{xu2020improving},Lemma 3)
    Consider the initialization distribution $\eta(\cdot)$ and transition kernel $\mathcal{P}(\cdot \mid s, a)$. Under $\mu_0(\cdot)$ and $\mathcal{P}(\cdot \mid s, a)$, denote $\mu_\pi(\cdot, \cdot)$ as the state-action visitation distribution of MDP with the Boltzmann policy parameterized by parameter $\theta$. For all policy parameter $\theta$ and $\theta^{\prime}$, we have
$$
\left\|\mu_\pi(\cdot, \cdot)-\mu_{\pi^{\prime}}(\cdot, \cdot)\right\|_{T V} \leq C_d\left\|\theta-\theta^{\prime}\right\|,
$$
where $C_d$ is a positive constant.
\end{lemma}

\begin{lemma}
Under the same reward function, we can establish the following bound of accumulated rewards under different policies:
\label{Lemma:sum_r_Q}
    \begin{equation}
    \label{reward_diff_policy}
    \begin{aligned}
        &\quad \mathbb{E}\left[\left\|\mathbb{E}_{\tau^A \sim \pi_{\theta_k}}\left[\sum_{t \geq 0} \gamma^t \nabla_\theta \widehat{r}\left(s_t, a_t ; \theta_k\right)\right]-\mathbb{E}_{\tau^{A^\prime} \sim \pi_{k+1}}\left[\sum_{t \geq 0} \gamma^t \nabla_\theta \widehat{r}\left(s_t, a_t ; \theta_k\right)\right]\right\|\right]\\
        &\leq  2 L_q C_d \mathbb{E}\left[\left\|\widehat{Q}_{\widehat{r}_{\theta_k}, \pi_{\theta_k}}^{\mathrm{soft}}-\widehat{Q}_{\widehat{r}_{\theta_k}, \pi_k}^{\mathrm{soft}}\right\|\right], 
    \end{aligned}
    \end{equation}
where the constant $L_q:=\frac{1}{\sqrt{m}(1-\gamma)}$.
\end{lemma}

Similar to the local linearization of the soft Q function, we have the following definition for the reward function.
\begin{definition}
    (Local Linearization of reward function) The neural network ${\widehat{r}}(x ; \theta)$  can be locally linearized at random initialization point $\theta_0$ by ${\widehat{r}}_0(x ; \theta)$ with respect to $\theta$:
    \begin{equation}
    \label{r0}
         \widehat{r}_0(x ; \theta)=\Psi(x)^{\top} \theta,
    \end{equation}
    $\text { where } \Psi(x)=\frac{1}{\sqrt{m}} \cdot\left(\mathbbm{1}\left\{{\theta_0^1}^{\top} x>0\right\} x, \ldots, \mathbbm{1}\left\{{\theta_0^m}^{\top} x>0\right\} x\right) \in \mathbb{R}^{m\times d}.$
\end{definition}

\begin{lemma}
 \label{gradient error bound}
    Let $V(\theta)= \widehat{r}\left(s_t, a_t;\theta \right)- \widehat{r}_{0}\left(s_t, a_t;\theta \right)$ be the error of local linearization of the reward function, then its gradient can be bounded as:
    \begin{equation}
        \nabla V(\theta)=\mathcal{O}\left(m^{-1 / 2}\right),
    \end{equation}
    $ \text{where }\widehat{r}_{0}\left(s_t, a_t;\theta \right) \text{ is defined in \eqref{r0}}$
\end{lemma}
We establish this lemma to characterize the gradient error bound of local linearization of the reward function, which is critical to the next proposition.


\begin{proposition}{Approximate Lipschitz Smoothness}
    \label{objective smoothness}
    \begin{equation}
    \label{Prop:L_smooth}
        \left\|\nabla \widehat{\mathcal{L}}\left(\theta_1\right)-\nabla \widehat{\mathcal{L}}\left(\theta_2\right)\right\| \leq 
\mathcal{O}\left(m^{-1 / 2}\right)+L_c\left\|\theta_1-\theta_2\right\|,
    \end{equation}
where $L_c=2L_qC_d \frac{1}{\sqrt{m}(1-\gamma)}$.
\end{proposition}

The Lipschitz smoothness property is common in the literature of min-max / bi-level optimization \cite{hong2020two,khanduri2021near,jin2020local,guan2021will,chen2021closing}. We get an approximate Lipschitz smoothness property by studying the local linearization of the neural network parametrized reward function.

\begin{lemma} 
\label{variance bound} (Lemma 4.5 in \cite{cai2023neural})
Recall that Bellman residual is calculated as $\delta_k:= z(k)= \widehat{Q}(s, a ; W_k)-\widehat{r}\left(s,a ;\theta_k\right)-\gamma \operatorname{softmax}_{a^{\prime} \in \mathcal{A}} \widehat{Q}\left(s^{\prime}, a^{\prime}; W_k\right)$. There exists $\sigma_z^2=O\left(B^2\right)$ such that the variance of the TD updates can  be upper bounded as follows:
\begin{equation}
    \mathbb{E}_{\text {init}, {\mu_k}}\left[\|z(k)\|_2^2\right] \leq \sigma_z^2,
\end{equation}
\begin{equation}
    \mathbb{E}_{\text {init}, {\mu_k}}\left[\|z(k)-\mathbb{E}_{\mu_k}[z(k)]\|_2^2\right] \leq \sigma_z^2,
\end{equation}
where the expectation is over initialization of parameter and the induced state-action distribution of current $\pi_k$.
\end{lemma}

\begin{lemma}
\label{Upper Bound of Stochastic Updates of Reward}
The stochastic gradient estimation $g_k:=h\left(\theta_k ; \tau_k^E\right)-h\left(\theta_k ; \tau_k^A\right)$, where $h(\theta ; \tau^E):=\sum_{t \geq 0} \gamma^t \nabla_\theta \widehat{r}\left(s_t, a_t ; \theta\right)$
\begin{equation}
    \left\|g_{k}\right\| \leq 2 L_q,
\end{equation}
\end{lemma}
This lemma provides a constant upper bound of stochastic updates of reward.\\
The following two lemmas are first established in \cite{cai2023neural} on studying the convergence of neural soft Q-learning. We'll use them in our lower-level analysis.
\begin{lemma}
Descent Lemma in soft Q-function (C.27 in \cite{cai2023neural})
\label{Descent Lemma in soft Q-function}

\begin{equation}
\label{eq in Descent Lemma}
    \begin{aligned}
& \quad \mathbb{E}{\left[\left(\widehat{Q}_0(x ; W_k)-\widehat{Q}_0\left(x ; W_{\theta_k}^*\right)\right)^2\right] } \\
& \leq \frac{\mathbb{E}\left[\left\|W_k-W_{\theta_k}^*\right\|_2^2\right]-\mathbb{E}\left[\left\|W_{k+1}-W_{\theta_k}^*\right\|_2^2\right]+\eta^2 \sigma_z^2}{\left(2 \eta_{k}(1-\gamma)-8 \eta^2\right)}+\mathcal{O}\left(B^3 m^{-1}+B^{5 / 2} m^{-1 / 2}\right),
\end{aligned}
\end{equation}
where the expectation is taken over all randomness and $\sigma_z$ is defined in Lemma \ref{variance bound}. $W_{\theta_k}^*$ is the parameter of the optimal soft Q-function under reward $\widehat{r}_{\theta_k}$.
\end{lemma}

\begin{lemma}
 Convergence bound of soft Q-learning (Theorem E.2 in \cite{cai2023neural})
\label{Convergence of Stochastic Update}
 We set $\eta$ to be of order $T^{-1 / 2}$ in Algorithm \ref{alg:1}. Under Assumptions \ref{regularity of policy} and \ref{Regularity of Stationary Distribution}, the output $\widehat{Q}_{\text {out }}$ of Algorithm \ref{alg:1} satisfies 
 \begin{equation}
        \mathbb{E}\left[\left(\widehat{Q}_{\text {out }}(x)-\widehat{Q}_0\left(x ; W_{\theta}^*\right)\right)^2\right]=\mathcal{O}\left(B^2 T^{-1 / 2}+B^3 m^{-1}+B^{5 / 2} m^{-1 /2}\right),
    \end{equation}
where T represents the number of iterations of Algorithm \ref{alg:1} and the expectation is over all randomness.
\end{lemma}

Note that the network error term here is slightly different from the original one in Theorem E.2 in \cite{cai2023neural}. This is because we adopt a different neural network parameterization.

\section{Proof of Auxiliary Lemmas}\label{Lemma_proof}
\subsection{Proof of Lemma \ref{objective reformulation}:}
First, our objective function $\widehat{\mathcal{L}}(\theta)$ in \eqref{L_theta} is defined as below:
$$
\widehat{\mathcal{L}}(\theta):=\mathbb{E}_{\tau^E \sim D}\left[\sum_{t=0}^{\infty} \gamma^t \log \pi_\theta\left(a_t \mid s_t\right)\right] \stackrel{(i)}{=} \mathbb{E}_{\tau^E \sim D}\left[\sum_{t=0}^{\infty} \gamma^t \log \left(\frac{\exp \left(\widehat{Q}_{\widehat{r}_\theta, \pi_\theta}^{\text {soft }}\left(s_t, a_t\right)\right)}{\sum_a \exp \left(\widehat{Q}_{\widehat{r}_\theta, \pi_\theta}^{\text {soft }}\left(s_t, a\right)\right)}\right)\right],
$$
where (i) is due to the fact that the optimal policy has the closed form $\pi_\theta(\cdot \mid s) \propto \exp \left(\widehat{Q}_{\widehat{r}_\theta, \pi_\theta}^{\text {soft }}(s, \cdot)\right)$. Therefore, the objective function can be expressed in the following form:
$$
\begin{aligned}
\widehat{\mathcal{L}}(\theta) & :=\mathbb{E}_{\tau^E \sim D}\left[\sum_{t=0}^{\infty} \gamma^t\left(\widehat{Q}_{\widehat{r}_\theta, \pi_\theta}^{\text {soft }}\left(s_t, a_t\right)-\log \left(\sum_a \exp \left(\widehat{Q}_{\widehat{r}_\theta, \pi_\theta}^{\text {soft }}\left(s_t, a\right)\right)\right)\right)\right] \\
& \stackrel{(i)}{=} \mathbb{E}_{\tau^E \sim D}\left[\sum_{t=0}^{\infty} \gamma^t\left(\widehat{Q}_{\widehat{r}_\theta, \pi_\theta}^{\text {soft }}\left(s_t, a_t\right)-V_{\widehat{r}_\theta, \pi_\theta}^{\text {soft }}\left(s_t\right)\right)\right] \\
& \stackrel{(ii)}=\mathbb{E}_{\tau^E \sim D}\left[\sum_{t=0}^{\infty} \gamma^t\left(\widehat{r}\left(s_t, a_t ; \theta\right)+\gamma V_{\widehat{r}_\theta, \pi_\theta}^{\text {soft }}\left(s_{t+1}\right)-V_{\widehat{r}_\theta, \pi_\theta}^{\text {soft }}\left(s_t\right)\right)\right] \\
& =\mathbb{E}_{\tau^E \sim D}\left[\sum_{t=0}^{\infty} \gamma^t \widehat{r}\left(s_t, a_t ; \theta\right)\right]+\mathbb{E}_{\tau^E \sim D}\left[\sum_{t=1}^{\infty} \gamma^t V_{\widehat{r}_\theta, \pi_\theta}^{\text {soft }}\left(s_t\right)\right]-\mathbb{E}_{\tau^E \sim D}\left[\sum_{t=0}^{\infty} \gamma^t V_{\widehat{r}_\theta, \pi_\theta}^{\text {soft }}\left(s_t\right)\right] \\
& =\mathbb{E}_{\tau^E \sim D}\left[\sum_{t=0}^{\infty} \gamma^t \widehat{r}\left(s_t, a_t ; \theta\right)\right]-\mathbb{E}_{s_0 \sim \mu_0(\cdot)}\left[V_{\widehat{r}_\theta, \pi_\theta}^{\text {soft }}\left(s_0\right)\right]\\
&\stackrel{(iii)}=\mathbb{E}_{\tau^E \sim D}\left[\sum_{t=0}^{\infty} \gamma^t \widehat{r}\left(s_t, a_t ; \theta\right)\right]-\mathbb{E}_{\tau^{\mathrm{A}} \sim \pi_\theta}\left[\sum_{t=0}^{\infty}\gamma^t \left(\widehat{r}\left(\mathbf{s}_t, \mathbf{a}_t;\theta\right)+\mathcal{H}\left(\pi\left(\cdot \mid \mathbf{s}_t\right)\right)\right)\right],
\end{aligned}
$$
where $(i)$  follows the fact that the optimal soft value function could be expressed as $V_{\widehat{r}_\theta, \pi_\theta}^{\text {soft }}(s)=$ $\log \left(\sum_a \exp \left(\widehat{Q}_{\widehat{r}_\theta, \pi_\theta}^{\text {soft }}(s, a)\right)\right)$. $(ii) \text{ and } (iii)$ is derived from the definition of soft Q value \eqref{2b} and soft V value \eqref{2a}.

\subsection{Proof of Lemma \ref{Lipschitz reward Lemma}}

We only prove the soft Q-function case here and the same reasoning also applies to reward function, since they employ the same parametrization structure. $\nabla_{W^j} \widehat{Q}(x ; W)=\frac{1}{{m}}b_j \mathbbm{1}\left\{{W^j}^{\top} x>0\right\} x$ almost everywhere, the $l_2$ norm is bounded as follows:
$$
\left\|\nabla_W \widehat{Q}(x ; W)\right\|_2\leq\frac{1}{{m}} \left(\sum_{j=1}^m \mathbbm{1}\left\{{W^j}^{\top} x>0\right\}\|b_j\|_2^2\|x\|_2^2\right)^{\frac{1}{2}} \leq \frac{1}{\sqrt{m}},
$$
Therefore, we can prove the lemma by the mean value theorem. Proof of inequality \eqref{soft q lip} is the same as \cite{zeng2022maximum} Lemma 5.3. 


\subsection{Proof of Lemma \ref{Lemma:sum_r_Q}}
 \begin{equation}
\label{L_smoothness_Q}
    \begin{aligned}
& \quad \mathbb{E}\left[\left\|\mathbb{E}_{\tau\sim \pi_{\theta_k}}\left[\sum_{t \geq 0} \gamma^t \nabla_\theta \widehat{r}\left(s_t, a_t ; \theta_k\right)\right]-\mathbb{E}_{\tau \sim \pi_{k+1}}\left[\sum_{t \geq 0} \gamma^t \nabla_\theta \widehat{r}\left(s_t, a_t ; \theta_k\right)\right]\right\|\right] \\
& \stackrel{(i)}{=} \mathbb{E}\left[\left\|\frac{1}{1-\gamma} \mathbb{E}_{(s, a) \sim d\left(\cdot, \cdot ; \pi_{\theta_k}\right)}\left[\nabla_\theta \widehat{r}\left(s, a ; \theta_k\right)\right]-\frac{1}{1-\gamma} \mathbb{E}_{(s, a) \sim d\left(\cdot, \cdot ; \pi_{k+1}\right)}\left[\nabla_\theta \widehat{r}\left(s, a ; \theta_k\right)\right]\right\|\right] \\
& \stackrel{(i i)}{\leq} \frac{2}{1-\gamma} \cdot \max _{s \in \mathcal{S}, a \in \mathcal{A}}\left\|\nabla_\theta \widehat{r}\left(s, a ; \theta_k\right)\right\| \cdot \mathbb{E}\left[\left\|d\left(\cdot, \cdot ; \pi_{\theta_k}\right)-d\left(\cdot, \cdot ; \pi_{k+1}\right)\right\|_{T V}\right] \\
& \stackrel{(i i i)}{\leq} \frac{2}{\sqrt{m}(1-\gamma)} \mathbb{E}\left[\left\|d\left(\cdot, \cdot ; \pi_{\theta_k}\right)-d\left(\cdot, \cdot ; \pi_{k+1}\right)\right\|_{T V}\right] \\
& \stackrel{(i v)}{\leq} 2 L_q C_d \mathbb{E}\left[\left\|\widehat{Q}_{\widehat{r}_{\theta_k}, \pi_{\theta_k}}^{\mathrm{soft}}-\widehat{Q}_{\widehat{r}_{\theta_k}, \pi_k}^{\mathrm{soft}}\right\|\right],
\end{aligned}
\end{equation}

where $(i)$ follows the definition $d(s, a; \pi)=(1-\gamma) \pi(a \mid s) \sum_{t \geq 0} \gamma^t \mathcal{P}^\pi\left(s_t=s\right)$; $(ii)$ is due to distribution mismatch between two visitation measures; $(iii)$ follows the inequality \eqref{r lip} in Lemma \ref{Lipschitz reward Lemma}; the inequality $(iv) $ follows Lemma \ref{lemma B.1} and the fact that $\pi_{\theta_k}(\cdot \mid s) \propto \exp \left(\widehat{Q}_{\widehat{r}_{\theta_k}, \pi_{\theta_k}}^{\text {soft }}(s, \cdot)\right), \pi_{k+1}(\cdot \mid s) \propto \exp \left(\widehat{Q}_{\widehat{r}_{\theta_k}, \pi_k}^{\text {soft }}(s, \cdot)\right)$, where the constant $L_q:=\frac{1}{\sqrt{m}(1-\gamma)}$. 

\subsection{Proof of Lemma \ref{gradient error bound}}

First, we prove $V(\theta)$ is Lipschitz continuous.
    \begin{equation*}
        \begin{aligned}
            \left\|V(\theta)-V(\theta^\prime)\right\|_{2}
            &\leq \left\|\widehat{r}(s_t, a_t;\theta)-\widehat{r}(s_t, a_t;\theta^\prime)\right\|_{2}+\left\|r_{0}(s_t, a_t;\theta)-r_{0}(s_t, a_t;\theta^\prime)\right\|_{2}\\
            &\stackrel{(\text { i } )}\leq \frac{2}{\sqrt{m}} \left\|\theta-\theta^{\prime}\right\|_2,
        \end{aligned}
    \end{equation*}
    where (i) is using the Lipshitz continuity property from \eqref{r lip} in Lemma \ref{Lipschitz reward Lemma}.\\
    Now we need to show $\nabla V(\theta)$ is bounded:
    $$\left|\nabla V(\theta)\right|=\lim _{\omega \rightarrow 0} \frac{|V(\theta+\omega)-V(\theta)|}{|\omega|} \leq \lim _{\omega \rightarrow 0} \frac{\frac{2}{\sqrt{m}}|\theta+\omega-\theta|}{|\omega|}=\frac{2}{\sqrt{m}},
    $$
Therefore, the lemma is shown.

\subsection{Proof of Proposition \ref{objective smoothness}}
By the triangle inequality, 
\begin{equation}
\label{r_0_proof}
\begin{aligned}
    \|\nabla_{\theta} \widehat{r}(x ; \theta_1)-\nabla_{\theta} \widehat{r}(x ; \theta_2)\|&\leq \underbrace{\|\nabla_{\theta} \widehat{r}(x ; \theta_1)-\nabla_{\theta} \widehat{r}_0(x ; \theta_1)\|}_{\text{Term I}}+\underbrace{\|\nabla_{\theta} \widehat{r}_0(x ; \theta_2)-\nabla_{\theta} \widehat{r}(x ; \theta_2)\|}_{\text{Term II}}\\
    &\quad +\underbrace{\|\nabla_{\theta} \widehat{r}_0(x ; \theta_1)-\nabla_{\theta} \widehat{r}_0(x ; \theta_2)\|}_{\text{Term III}}.
    \end{aligned}
\end{equation}
 Both Term I and Term II are bounded by $\mathcal{O}(m^{-\frac{1}{2}})$ in Lemma \ref{gradient error bound}.

For Term III, recall that $\nabla_{\theta} \widehat{r}_0(x ; \theta)=\frac{1}{{m}}b\mathbbm{1}\left\{{\theta_0}^{\top} x>0\right\} x$. The following inequality holds:
\begin{equation}
\label{r_0_smooth}
    \|\nabla_{\theta} \widehat{r}_0(x ; \theta_1)-\nabla_{\theta} \widehat{r}_0(x ; \theta_2)\|=0.
\end{equation}
We then plug \eqref{r_0_smooth} into \eqref{r_0_proof} and yield the following inequality:
 \begin{equation}
 \label{r_smooth}
        \|\nabla_{\theta} \widehat{r}(x ; \theta_1)-\nabla_{\theta} \widehat{r}(x ; \theta_2)\|\leq 
\mathcal{O}\left(m^{-1 / 2}\right).
    \end{equation}

By Lemma \ref{L gradient}, we have:\\
$\nabla \widehat{\mathcal{L}}(\theta)= \mathbb{E}_{\tau^E \sim \mathcal{D}}\left[\sum_{t \geq 0} \gamma^t \nabla_\theta \widehat{r}\left(s_t, a_t ; \theta\right)\right]-\mathbb{E}_{\tau^{\mathrm{A}} \sim \pi_\theta}\left[\sum_{t \geq 0} \gamma^t \nabla_\theta \widehat{r}\left(s_t, a_t ; \theta\right)\right] .$

Therefore, 
\begin{equation}
    \begin{aligned}
& \quad \left\|\nabla \widehat{\mathcal{L}}\left(\theta_1\right)-\nabla \widehat{\mathcal{L}}\left(\theta_2\right)\right\| \\
&\leq \underbrace{\left\|\mathbb{E}_{\tau^E \sim D}\left[\sum_{t \geq 0} \gamma^t \nabla_\theta \widehat{r}\left(s_t, a_t ; \theta_1\right)\right]-\mathbb{E}_{\tau^E \sim D}\left[\sum_{t \geq 0} \gamma^t \nabla_\theta \widehat{r}\left(s_t, a_t ; \theta_2\right)\right]\right\|+}_{:=\text {term } \mathrm{A}} \\
& \quad \underbrace{\left\|\mathbb{E}_{\tau^A \sim \pi_{\theta_1}}\left[\sum_{t \geq 0} \gamma^t \nabla_\theta \widehat{r}\left(s_t, a_t ; \theta_1\right)\right]-\mathbb{E}_{\tau^{A^\prime} \sim \pi_{\theta_2}}\left[\sum_{t \geq 0} \gamma^t \nabla_\theta \widehat{r}\left(s_t, a_t ; \theta_2\right)\right]\right\|}_{:=\text {term B }}.
\end{aligned}
\end{equation}

For term A, it follows that:
$$
\begin{aligned}
& \quad \left\|\mathbb{E}_{\tau^E \sim D}\left[\sum_{t \geq 0} \gamma^t \nabla_\theta \widehat{r}\left(s_t, a_t ; \theta_1\right)\right]-\mathbb{E}_{\tau^E \sim D}\left[\sum_{t \geq 0} \gamma^t \nabla_\theta \widehat{r}\left(s_t, a_t ; \theta_2\right)\right]\right\| \\
& \stackrel{(i)}{\leq} \mathbb{E}_{\tau^E \sim D}\left[\sum_{t \geq 0} \gamma^t\left\|\nabla_\theta \widehat{r}\left(s_t, a_t ; \theta_1\right)-\nabla_\theta \widehat{r}\left(s_t, a_t ; \theta_2\right)\right\|\right] \\
&\stackrel{(i i)}\leq \mathcal{O}(m^{-1/2}),
\end{aligned}
$$
where $(i)$ follows the Jensen's inequality and $(ii)$ is from inequality \eqref{r_smooth}.

For term B,
$$
\begin{aligned}
&\quad \left\|\mathbb{E}_{\tau \sim \pi_{\theta_1}}\left[\sum_{t \geq 0} \gamma^t \nabla_\theta \widehat{r}\left(s_t, a_t ; \theta_1\right)\right]-\mathbb{E}_{\tau \sim \pi_{\theta_2}}\left[\sum_{t \geq 0} \gamma^t \nabla_\theta \widehat{r}\left(s_t, a_t ; \theta_2\right)\right]\right\| \\
& \stackrel{(i)}{\leq}\left\|\mathbb{E}_{\tau \sim \pi_{\theta_1}}\left[\sum_{t \geq 0} \gamma^t \nabla_\theta \widehat{r}\left(s_t, a_t ; \theta_1\right)\right]-\mathbb{E}_{\tau \sim \pi_{\theta_2}}\left[\sum_{t \geq 0} \gamma^t \nabla_\theta \widehat{r}\left(s_t, a_t ; \theta_1\right)\right]\right\| \\
&\quad +\left\|\mathbb{E}_{\tau \sim \pi_{\theta_2}}\left[\sum_{t \geq 0} \gamma^t \nabla_\theta \widehat{r}\left(s_t, a_t ; \theta_1\right)\right]-\mathbb{E}_{\tau \sim \pi_{\theta_2}}\left[\sum_{t \geq 0} \gamma^t \nabla_\theta \widehat{r}\left(s_t, a_t ; \theta_2\right)\right]\right\| \\
& \stackrel{(i i)}{\leq} \left\|\mathbb{E}_{\tau \sim \pi_{\theta_1}}\left[\sum_{t \geq 0} \gamma^t \nabla_\theta \widehat{r}\left(s_t, a_t ; \theta_1\right)\right]-\mathbb{E}_{\tau \sim \pi_{\theta_2}}\left[\sum_{t \geq 0} \gamma^t \nabla_\theta \widehat{r}\left(s_t, a_t ; \theta_1\right)\right]\right\|  \\
&\quad  +\mathbb{E}_{\tau \sim \pi_{\theta_2}}\left[\sum_{t \geq 0} \gamma^t\left\|\nabla_\theta \widehat{r}\left(s_t, a_t ; \theta_1\right)-\nabla_\theta \widehat{r}\left(s_t, a_t ; \theta_2\right)\right\|\right] \\
& \stackrel{(iii)}{\leq}  2L_qC_d\left\|\widehat{Q}_{\widehat{r}_{\theta_1}, \pi_{\theta_1}}^{\text {soft }}-\widehat{Q}_{\widehat{r}_{\theta_2}, \pi_{\theta_2}}^{\text {soft }}\right\|_{2}+\mathcal{O}(m^{-1/2})\\
& \stackrel{(iv)}{\leq}  2L_qC_d \frac{1}{\sqrt{m}(1-\gamma)}\left\|\theta_1-\theta_2\right\|_{2}+\mathcal{O}(m^{-1/2}),
\end{aligned}
$$

where $(i)$ follows the triangle inequality, $(ii)$ is from Jensen's inequality ; $(iii)$ is from \eqref{r_smooth} and Lemma \ref{Lemma:sum_r_Q}; $(iv)$ follows inequality \eqref{soft q lip} in Lemma \ref{Lipschitz reward Lemma}.

Combining term A and term B, we obtain the approximate L-smoothness property:
\begin{equation}
    \left\|\nabla \widehat{\mathcal{L}}\left(\theta_1\right)-\nabla \widehat{\mathcal{L}}\left(\theta_2\right)\right\| \leq 
\mathcal{O}\left(m^{-1 / 2}\right)+L_c\left\|\theta_1-\theta_2\right\|,
\end{equation}

where $L_c=2L_qC_d \frac{1}{\sqrt{m}(1-\gamma)}$.

\subsection{Proof of Lemma \ref{Upper Bound of Stochastic Updates of Reward}}
$\left\|g_{k}\right\| \leq\left\|h\left(\theta_{k}, \tau^E_{k}\right)\right\|+\left\|h\left(\theta_{k}, \tau_{k}^A\right)\right\| \leq 2 \frac{1}{\sqrt{m}} \sum_{t \geq 0} \gamma^t=\frac{2 }{\sqrt{m}(1-\gamma)}=2 L_q$.

\section{Lower-level Analysis}

\begin{lemma}
\label{Lemma:pileqQ}
Given that the policies $\pi_{k+1}$ and $\pi_{\theta_k}$ are in the softmax parameterization, we can establish an upper bound for the difference between $\pi_{k+1}$ and $\pi_{\theta_k}$ as follows:
    \begin{equation}
    \left\|\log \pi_{k+1}-\log \pi_{\theta_k}\right\|_{\infty} \leq 2\left\|\widehat{Q}_{\widehat{r}_{\theta_k}, \pi_k}^{\text {soft }}-\widehat{Q}_{\widehat{r}_{\theta_k}, \pi_{\theta_k}}^{\text {soft }}\right\|_{2}
\end{equation}
\end{lemma}

\begin{proof}
    We first denote an operator $\log \left(\|\exp (v)\|_1\right):=\log \left(|\sum_{\tilde{a} \in \mathcal{A}} \exp \left(v_{\tilde{a}}\right)|\right)$, where the vector $v \in \mathbb{R}^{|\mathcal{A}|}$ and $v=$ $\left[v_1, v_2, \cdots, v_{|\mathcal{A}|}\right]$. Then for any $v^{\prime}, v^{\prime \prime} \in \mathbb{R}^{|\mathcal{A}|}$, we have the following relation:

\begin{equation}
\label{log_exp_diff}
    \begin{aligned}
 \log \left(\left\|\exp \left(v^{\prime}\right)\right\|_1\right)-\log \left(\left\|\exp \left(v^{\prime \prime}\right)\right\|_1\right) & \stackrel{(i)}{=}\left\langle v^{\prime}-v^{\prime \prime},\left.\nabla_v \log \left(\|\exp (v)\|_1\right)\right|_{v=v_c}\right\rangle \\
& \leq\left\|v^{\prime}-v^{\prime \prime}\right\|_{\infty} \cdot\left\|\left.\nabla_v \log \left(\|\exp (v)\|_{1}\right)\right|_{v=v_c}\right\|_1 \\
& \stackrel{(i i)}{=}\left\|v^{\prime}-v^{\prime \prime}\right\|_{\infty},
\end{aligned}
\end{equation}
where (i) follows the mean value theorem and $v_c$ is a convex combination of $v^{\prime}$ and $v^{\prime \prime}$, and (ii) follows the following equalities:
$$
\left[\nabla_v \log \left(\|\exp (v)\|_1\right)\right]_i=\frac{\exp \left(v_i\right)}{\sum_{1 \leq a \leq|\mathcal{A}|} \exp \left(v_a\right)}, \quad\left\|\nabla_v \log \left(\|\exp (v)\|_1\right)\right\|_1=1, \quad \forall v \in \mathbb{R}^{|\mathcal{A}|} .
$$

For any $s \in \mathcal{S}$ and $a \in \mathcal{A}$, we have the following relationship:
\begin{equation}
\label{log_pi_diff}
    \begin{aligned}
&\quad \left|\log \left(\pi_{k+1}(a \mid s)\right)-\log \left(\pi_{\theta_k}(a \mid s)\right)\right| \\
& \stackrel{(i)}{=}\left|\log \left(\frac{\exp \left(\widehat{Q}_{\widehat{r}_{\theta_k}, \pi_k}^{\text {soft }}(s, a)\right)}{\sum_{\tilde{a}} \exp \left(\widehat{Q}_{\widehat{r}_{\theta_k}, \pi_k}^{\mathrm{soft}}(s, \tilde{a})\right)}\right)-\log \left(\frac{\exp \left(\widehat{Q}_{\widehat{r}_{\theta_k}, \pi_{\theta_k}}^{\text {soft }}(s, a)\right)}{\sum_{\tilde{a}} \exp \left(\widehat{Q}_{\widehat{r}_{\theta_k}, \pi_{\theta_k}}^{\mathrm{soft}}(s, \tilde{a})\right)}\right)\right| \\
& \stackrel{(i i)}{\leq}\left|\log \left(\sum_{\tilde{a}} \exp \left(\widehat{Q}_{\widehat{r}_{\theta_k}, \pi_k}^{\mathrm{soft}}(s, \tilde{a})\right)\right)-\log \left(\sum_{\tilde{a}} \exp \left(\widehat{Q}_{\widehat{r}_{\theta_k}, \pi_{\theta_k}}^{\text {soft }}(s, \tilde{a})\right)\right)\right|\\
&\quad +\left|\widehat{Q}_{\widehat{r}_{\theta_k}, \pi_k}^{\text {soft }}(s, a)-\widehat{Q}_{\widehat{r}_{\theta_k}, \pi_{\theta_k}}^{\text {soft }}(s, a)\right|,
\end{aligned}
\end{equation}
where $(i)$ is from $\pi(a \mid s) \propto \exp \left(Q^{\text {soft }}(s, a)\right)$ and $(ii)$ is the result of triangle inequality.

We plug \eqref{log_exp_diff} into \eqref{log_pi_diff} and yield that
\begin{equation}
    \begin{aligned}
& \quad \left|\log \left(\pi_{k+1}(a \mid s)\right)-\log \left(\pi_{\theta_k}(a \mid s)\right)\right| \\
&\leq \left|\widehat{Q}_{\widehat{r}_{\theta_k}, \pi_k}^{\mathrm{soft}}(s, a)-\widehat{Q}_{\widehat{r}_{\theta_k}, \pi_{\theta_k}}^{\text {soft }}(s, a)\right|+\max _{\tilde{a} \in \mathcal{A}}\left|\widehat{Q}_{\widehat{r}_{\theta_k}, \pi_k}^{\text {soft }}(s, \tilde{a})-\widehat{Q}_{\widehat{r}_{\theta_k}, \pi_{\theta_k}}^{\mathrm{soft}}(s, \tilde{a})\right|.
\end{aligned}
\end{equation}

Taking the infinity norm over $\mathbb{R}^{|\mathcal{S}| \cdot|\mathcal{A}|}$ on both sides, the following result holds:
\begin{equation}
\label{policy inequality}
    \left\|\log \pi_{k+1}-\log \pi_{\theta_k}\right\|_{\infty} \leq 2\left\|\widehat{Q}_{\widehat{r}_{\theta_k}, \pi_k}^{\text {soft }}-\widehat{Q}_{\widehat{r}_{\theta_k}, \pi_{\theta_k}}^{\text {soft }}\right\|_{\infty},
\end{equation}

where $\left\|\log \pi_{k+1}-\log \pi_{\theta_k}\right\|_{\infty} =\max _{s \in \mathcal{S}, a \in \mathcal{A}}\left|\log \pi_{k+1}(a \mid s)-\log \pi_{\theta_k}(a \mid s)\right|$ and $\left\|\widehat{Q}_{\widehat{r}_{\theta_k}, \pi_k}^{\text {soft }}-\widehat{Q}_{\widehat{r}_{\theta_k}, \pi_{\theta_k}}^{\text {soft }}\right\|_{\infty}=$ $\max _{s \in \mathcal{S}, a \in \mathcal{A}}\left|\widehat{Q}_{\widehat{r}_{\theta_k}, \pi_k}^{\text {soft }}(s, a)-\widehat{Q}_{\widehat{r}_{\theta_k}, \pi_{\theta_k}}^{\text {soft }}(s, a)\right|.$

By the definition of supremum norm and $l_2$ norm, we have the following relation:
\begin{equation}
\label{sup<l2}
\left\|\widehat{Q}_{\widehat{r}_{\theta_k}, \pi_k}^{\text {soft }}-\widehat{Q}_{\widehat{r}_{\theta_k}, \pi_{\theta_k}}^{\text {soft }}\right\|_{\infty}
        \leq \left\|\widehat{Q}_{\widehat{r}_{\theta_k}, \pi_k}^{\text {soft }}-\widehat{Q}_{\widehat{r}_{\theta_k}, \pi_{\theta_k}}^{\text {soft }}\right\|_{2}.       
\end{equation}
Combining the inequality \eqref{policy inequality} and \eqref{sup<l2}, we prove the lemma and now only need to analyze $\left\|\widehat{Q}_{\widehat{r}_{\theta_k}, \pi_k}^{\text {soft }}-\widehat{Q}_{\widehat{r}_{\theta_k}, \pi_{\theta_k}}^{\text {soft }}\right\|_{2}$ to show the convergence of the policy parameter.
\end{proof}

\begin{theorem}
\label{Projection Theorem for Convex Sets}
     (Norm of the Projection Operator). Let $v \in \mathbb{R}^n$ and $\Pi_{\mathcal{F}_{B, m}}$ projects $v$ onto the subspace $\mathcal{F}_{B, m}$, 
$$
\left\|\Pi_{\mathcal{F}_{B, m}} v\right\|_2 \leq\|v\|_2.
$$

\end{theorem}

The theorem can be easily verified by decomposing $v$ into $v=\Pi_{\mathcal{F}_{B, m}} v+r$, where $r$ is orthogonal to $\Pi_{\mathcal{F}_{B, m}} v$.
\vspace{-5pt}
\subsection{Algorithm 2}
\label{thm 1 proof}
\vspace{-5pt}
We set  $\eta_{k}=\min \{1 / \sqrt{K}, \frac{1-\gamma}{8} \}$, $
\text { Note that when } K \geq(8 /(1-\gamma))^2 \text {, we have } \eta_{k}=\frac{1}{\sqrt{K}}
.$ 
Therefore, 

$$
\begin{aligned}
\frac{1}{K} \sum_{k=0}^{K-1}\left\|\widehat{Q}_{\widehat{r}_{\theta_k}, \pi_k}^{\text {soft }}-\widehat{Q}_{\widehat{r}_{\theta_k}, \pi_{\theta_k}}^{\text {soft }}\right\|_{2} &\stackrel{(i)} \leq \frac{B}{K}\sum_{k=0}^{K-1} (k+2)^{-1 / 4}+\mathcal{O}\left(B^{3/2} m^{-1 / 2}+B^{5 / 4} m^{-1 / 4}\right) \\
 &\leq  \frac{B}{K} \int_{k=0}^{K-1}(k+1)^{-\frac{1}{4}}+\mathcal{O}\left(B^{3/2} m^{-1 / 2}+B^{5 / 4} m^{-1 / 4}\right)\\
 &\leq BK^{-\frac{1}{4}}+\mathcal{O}\left(B^{3/2} m^{-1 / 2}+B^{5 / 4} m^{-1 / 4}\right),
\end{aligned}
$$
 where $(i)$ comes from the dynamic truncation design that $T=k+2$ in Algorithm \ref{alg:1} and Lemma \ref{Convergence of Stochastic Update}.
 Therefore, we have:
 \begin{equation}
 \label{Alg2_Q_Bound}
     \frac{1}{K} \sum_{k=0}^{K-1}\left\|\widehat{Q}_{\widehat{r}_{\theta_k}, \pi_k}^{\text {soft }}-\widehat{Q}_{\widehat{r}_{\theta_k}, \pi_{\theta_k}}^{\text {soft }}\right\|_{2}  \leq BK^{-\frac{1}{4}}+\mathcal{O}\left(B^{3/2} m^{-1 / 2}+B^{5 / 4} m^{-1 / 4}\right).
 \end{equation}
 
Following \eqref{policy inequality}, we obtain that 
\begin{equation}
    \label{algorithm 2 policy}
    \frac{1}{K} \sum_{k=0}^{K-1}\left\|\log \pi_{k+1}-\log \pi_{\theta_k}\right\|_{\infty} =\mathcal{O}(K^{-\frac{1}{4}})+\mathcal{O}\left(B^{3/2} m^{-1 / 2}+B^{5 / 4} m^{-1 / 4}\right).
\end{equation}

\subsection{Algorithm 3}
\label{A3_policy_proof}
Telescope \eqref{eq in Descent Lemma} in Lemma \ref{Descent Lemma in soft Q-function} \text { for } k=0, \ldots, K-1,

\begin{equation}
\label{soft Q telescope}
\begin{aligned}
   \frac{1}{K} \sum_{k=0}^{K-1}\left\|\widehat{Q}_{\widehat{r}_{\theta_k}, \pi_k}^{\text {soft }}-\widehat{Q}_{\widehat{r}_{\theta_k}, \pi_{\theta_k}}^{\text {soft }}\right\|^2_{2}  & \leq \frac{1}{K} \sum_{k=0}^{K-1} \frac{\big[\mathbb{E}\left[\left\|W_k-W_{\pi_k}^*\right\|_2^2\right]-\mathbb{E}\left[\left\|W_{k+1}-W_{\pi_{k}}^*\right\|_2^2\right]+\eta^2 \sigma_z^2\big]}{2\eta (1-\gamma)-8\eta^2}\\
& \quad +\mathcal{O}\left(B^3 m^{-1 }+B^{5 / 2} m^{-1 / 2}\right).
\end{aligned}
\end{equation}

For each $k$, we apply triangle inequality on the right-hand side of \eqref{soft Q telescope} and yield the following inequality:

\begin{equation}
\label{soft Q telescope+triangle}
\frac{1}{K} \sum_{k=0}^{K-1}\left\|\widehat{Q}_{\widehat{r}_{\theta_k}, \pi_k}^{\text {soft }}-\widehat{Q}_{\widehat{r}_{\theta_k}, \pi_{\theta_k}}^{\text {soft }}\right\|^2_{2} \leq \frac{1}{K} \sum_{k=0}^{K-1} \frac{\mathbb{E}\left[\left\|W_k-W_{k+1}\right\|_2^2\right]+\eta^2 \sigma_z^2}{2\eta (1-\gamma)-8\eta^2}+\mathcal{O}\left(B^3 m^{-1 }+B^{5 / 2} m^{-1 / 2}\right)
\end{equation}

Recall that projection step (line 6) $W_{k+1}=\Pi_{\mathcal{F}_{B, m}}\widetilde{W}_{k+1}$ in Algorithm \ref{alg:3}, we apply Theorem \ref{Projection Theorem for Convex Sets} to \eqref{soft Q telescope+triangle} and the following relation holds:
\begin{equation}
\label{soft Q telescope+triangle+Projection}
\frac{1}{K} \sum_{k=0}^{K-1}\left\|\widehat{Q}_{\widehat{r}_{\theta_k}, \pi_k}^{\text {soft }}-\widehat{Q}_{\widehat{r}_{\theta_k}, \pi_{\theta_k}}^{\text {soft }}\right\|^2_{2} \leq  \frac{1}{K} \sum_{k=0}^{K-1} \frac{\mathbb{E}\left[\left\| W_k-\widetilde W_{k+1}\right\|_2^2\right]+\eta^2 \sigma_z^2}{2\eta (1-\gamma)-8\eta^2}+\mathcal{O}\left(B^3 m^{-1 }+B^{5 / 2} m^{-1 / 2}\right).
\end{equation}

By the TD update (line 5) of the Algorithm \ref{alg:3}, we know that 
\begin{equation}
\label{TD update of algorithm 3}
    \widetilde{W}_{k+1}= W_k-\eta \delta \cdot \nabla_W \widehat{Q}(s, a ; W_k).
\end{equation}

Plug \eqref{TD update of algorithm 3}
into \eqref{soft Q telescope+triangle+Projection}, we get that:
\begin{equation*}
\begin{aligned}
\frac{1}{K} \sum_{k=0}^{K-1}\left\|\widehat{Q}_{\widehat{r}_{\theta_k}, \pi_k}^{\text {soft }}-\widehat{Q}_{\widehat{r}_{\theta_k}, \pi_{\theta_k}}^{\text {soft }}\right\|^2_{2} &\leq \frac{1}{K} \sum_{k=0}^{K-1} \frac{\mathbb{E}\left[\eta^2\delta_k^2\left\|\nabla_W \widehat{Q}(s, a ; W_k)\right\|_2^2\right]+\eta^2 \sigma_z^2}{2\eta (1-\gamma)-8\eta^2}+\mathcal{O}\left(B^3 m^{-1 }+B^{5 / 2} m^{-1 / 2}\right)\\
 &  \stackrel{(i)} \leq \frac{B^2+\sigma^2_z}{K} \sum_{k=0}^{K-1}\frac{\eta^2 }{2\eta (1-\gamma)-8\eta^2} +\mathcal{O}\left(B^3 m^{-1 }+B^{5 / 2} m^{-1 / 2}\right),
\end{aligned}
\end{equation*}

where (i) comes from Lemma \ref{variance bound}.

If we set the stepsize $\eta \leq \frac{1}{\sqrt{K}}$, we obtain the following result:
\begin{equation}
\label{soft Q telescope+triangle+Projection+stepsize}
 \frac{1}{K} \sum_{k=0}^{K-1}\left\|\widehat{Q}_{\widehat{r}_{\theta_k}, \pi_k}^{\text {soft }}-\widehat{Q}_{\widehat{r}_{\theta_k}, \pi_{\theta_k}}^{\text {soft }}\right\|^2_{2}   \leq \frac{B^2+\sigma^2_z}{K} \frac{1 }{2\eta (1-\gamma)-8\eta^2}+\mathcal{O}\left(B^3 m^{-1}+B^{5 / 2} m^{-1 / 2}\right).
\end{equation}
We select $\eta=\min \{K^{-\frac{3}{4}},(1-\gamma) / 8\}$, which satisfies the condition $\eta \leq \frac{1}{\sqrt{K}}$ required in \eqref{soft Q telescope+triangle+Projection+stepsize}. Note that when $K \geq(8 /(1-\gamma))^{\frac{4}{3}}$, we have $\eta=K^{-\frac{3}{4}}$ and
\begin{equation*}
   K^{\frac{3}{4}}\cdot\left(2 \eta(1-\gamma)-8 \eta^2\right)=2(1-\gamma)-8 K^{-\frac{3}{4}} \geq 1-\gamma . 
\end{equation*}

When $K <(8 /(1-\gamma))^{\frac{4}{3}}$, we have $\eta=(1-\gamma) / 8$ and
\begin{equation*}
    K^{\frac{3}{4}} \cdot\left(2 \eta(1-\gamma)-8 \eta^2\right)=K^{\frac{3}{4}} \cdot(1-\gamma)^2 / 8 \geq(1-\gamma)^2 / 8 .
\end{equation*}
Since $|1-\gamma|<1$, we obtain that for any $K \in \mathbb{N}$,
\begin{equation}
\label{convergence critical bound}
    \frac{1}{K^{\frac{3}{4}} \cdot\left(2 \eta(1-\gamma)-8 \eta^2\right)} \leq \frac{8}{(1-\gamma)^2}.
\end{equation}
Then, 
\begin{equation}
\label{soft Q final bound}
     \frac{1}{K} \sum_{k=0}^{K-1}\left\|\widehat{Q}_{\widehat{r}_{\theta_k}, \pi_k}^{\text {soft }}-\widehat{Q}_{\widehat{r}_{\theta_k}, \pi_{\theta_k}}^{\text {soft }}\right\|^2_{2}   \leq \frac{B^2+\sigma^2_z}{K^{\frac{1}{4}}} \frac{8 }{ (1-\gamma)^2}+\mathcal{O}\left(B^3 m^{-1}+B^{5 / 2} m^{-1 / 2}\right).
\end{equation}

Following Lemma \ref{Lemma:pileqQ}, we have:
\begin{equation}
    \begin{aligned}
    \frac{1}{K} \sum_{k=0}^{K-1}\left\|\log \pi_{k+1}-\log \pi_{\theta_k}\right\|_{\infty} &\leq \frac{2}{K} \sum_{k=0}^{K-1}\left\|\widehat{Q}_{\widehat{r}_{\theta_k}, \pi_k}^{\text {soft }}-\widehat{Q}_{\widehat{r}_{\theta_k}, \pi_{\theta_k}}^{\text {soft }}\right\|_{2}\\
 & \stackrel{(i)}\leq 2 \sqrt{\frac{1}{K} \sum_{k=0}^{K-1}\left\|\widehat{Q}_{\widehat{r}_{\theta_k}, \pi_k}^{\text {soft }}-\widehat{Q}_{\widehat{r}_{\theta_k}, \pi_{\theta_k}}^{\text {soft }}\right\|^2_{2} }\\
 &\stackrel{(ii)}\leq 2 \sqrt{\frac{B^2+\sigma^2_z}{K} \frac{1}{2\eta (1-\gamma)-8\eta^2}+\mathcal{O}\left(B^3 m^{-1 }+B^{5 / 2} m^{-1 / 2}\right)}\\
 &\stackrel{(iii)} \leq 2 \sqrt{B^2+\sigma^2_z}\sqrt{\frac{1}{K^{\frac{1}{4}}}\frac{8}{(1-\gamma)^2}}+\mathcal{O}\left(B^{3/2} m^{-1 / 2}+B^{5 / 4} m^{-1 / 4}\right),
\end{aligned}
\end{equation}
where (i) comes from the Cauchy-Schwartz inequality, (ii) is due to \eqref{soft Q final bound}, and (iii) is given by  \eqref{convergence critical bound}.
Therefore, we obtain the convergence guarantee of policy parameter for Algorithm \ref{alg:3}:
\begin{equation}
    \frac{1}{K} \sum_{k=0}^{K-1}\left\|\log \pi_{k+1}-\log \pi_{\theta_k}\right\|_{\infty} = \mathcal{O}(K^{-\frac{1}{8}})+\mathcal{O}\left(B^{3/2} m^{-1 / 2}+B^{5 / 4} m^{-1 / 4}\right).
\end{equation}

\section{Convergence of reward parameters}
Since Proposition \ref{objective smoothness} establishes the approximate Lipschitz smooth property for the objective function, we have the following result of $\widehat{\mathcal{L}}(\theta)$:
\begin{equation}
\label{L_smooth_ori}
    \widehat{\mathcal{L}}\left(\theta_{k+1}\right)+\mathcal{O}\left(m^{-1/2}\right) \geq \widehat{\mathcal{L}}\left(\theta_K\right)+\left\langle\nabla \widehat{\mathcal{L}}\left(\theta_K\right), \theta_{k+1}-\theta_k\right\rangle-\frac{L_c}{2}\left\|\theta_{k+1}-\theta_k\right\|^2.
\end{equation}

Notice the reward update rule (line 13) of the Algorithm \ref{alg:3}, 
\begin{equation}
\label{theta update rule}
    \theta_{k+1}:=\theta_k+\alpha g_k.
\end{equation}

We plug \eqref{theta update rule} into \eqref{L_smooth_ori} and get:
\begin{equation}
\label{L+smooth+update}
    \begin{aligned}
\widehat{\mathcal{L}}\left(\theta_{k+1}\right)+\mathcal{O}\left(m^{-1/2}\right) &\geq \widehat{\mathcal{L}}\left(\theta_K\right)+\alpha\left(\nabla \widehat{\mathcal{L}}\left(\theta_K\right), g_k\right\rangle-\frac{L_c \alpha^2}{2}\left\|g_k\right\|^2\\
& =\widehat{\mathcal{L}}\left(\theta_K\right)+\alpha\left(\nabla \widehat{\mathcal{L}}\left(\theta_K\right), g_k-\nabla \widehat{\mathcal{L}}\left(\theta_K\right)\right\rangle+\alpha\left\|\nabla \widehat{\mathcal{L}}\left(\theta_K\right)\right\|^2-\frac{L_c \alpha^2}{2}\left\|g_k\right\|^2 \\
& \stackrel{(i)}{\geq} \widehat{\mathcal{L}}\left(\theta_K\right)+\alpha\left(\nabla \widehat{\mathcal{L}}\left(\theta_K\right), g_k-\nabla \widehat{\mathcal{L}}\left(\theta_K\right)\right\rangle+\alpha\left\|\nabla \widehat{\mathcal{L}}\left(\theta_K\right)\right\|^2-2 L_c L_q^2 \alpha^2,
\end{aligned}
\end{equation}
where (i) is from the upper bound of stochastic reward updates in Lemma \ref{Upper Bound of Stochastic Updates of Reward}.

Taking an expectation over both sides of \eqref{L+smooth+update}, it holds that
\begin{equation}
\label{L_smooth_exp}
    \begin{aligned}
&\quad \mathbb{E}\left[\widehat{\mathcal{L}}\left(\theta_{k+1}\right)\right]+\mathcal{O}\left(m^{-1/2}\right)\\
& \geq \mathbb{E}\left[\widehat{\mathcal{L}}\left(\theta_K\right)\right]+\alpha \mathbb{E}\left[\left\langle\nabla \widehat{\mathcal{L}}\left(\theta_K\right), g_k-\nabla \widehat{\mathcal{L}}\left(\theta_K\right)\right\rangle\right]+\alpha \mathbb{E}\left[\left\|\nabla \widehat{\mathcal{L}}\left(\theta_K\right)\right\|^2\right]-2 L_c L_q^2 \alpha^2 \\
& =\mathbb{E}\left[\widehat{\mathcal{L}}\left(\theta_K\right)\right]+\alpha \mathbb{E}\left[\left\langle\nabla \widehat{\mathcal{L}}\left(\theta_K\right), \mathbb{E}\left[g_k-\nabla \widehat{\mathcal{L}}\left(\theta_K\right) \mid \theta_k\right]\right\rangle\right]+\alpha \mathbb{E}\left[\left\|\nabla \widehat{\mathcal{L}}\left(\theta_K\right)\right\|^2\right]-2 L_c L_q^2 \alpha^2. \\
\end{aligned}
\end{equation}

Combining the definition of $g_k$ in line 12 of Algorithm \ref{alg:3} and Lemma \ref{L gradient} into \eqref{L_smooth_exp}, we have:
\begin{equation*}
\begin{aligned}
&\quad \mathbb{E}\left[\widehat{\mathcal{L}}\left(\theta_{k+1}\right)\right]+\mathcal{O}\left(m^{-1/2}\right)\\
& \geq \mathbb{E}\left[\widehat{\mathcal{L}}\left(\theta_K\right)\right]+\alpha \mathbb{E}\left[\left\|\nabla \widehat{\mathcal{L}}\left(\theta_K\right)\right\|^2\right]-2 L_c L_q^2 \alpha^2 \\
& \quad +\alpha \mathbb{E}\left[\left\langle\nabla \widehat{\mathcal{L}}\left(\theta_K\right), \mathbb{E}_{\tau^E \sim \pi_{\theta_k}}\left[\sum_{t \geq 0} \gamma^t \nabla_\theta \widehat{r}\left(s_t, a_t ; \theta_t\right)\right]-\mathbb{E}_{\tau^E \sim \pi_{k+1}}\left[\sum_{t \geq 0} \gamma^t \nabla_\theta \widehat{r}\left(s_t, a_t ; \theta_t\right)\right]\right\rangle\right]\\
& \stackrel{(i)}{\geq} \mathbb{E}\left[\widehat{\mathcal{L}}\left(\theta_K\right)\right]-2 \alpha L_q \mathbb{E}\left[\left\|\mathbb{E}_{\tau^E \sim \pi_{\theta_k}}\left[\sum_{t \geq 0} \gamma^t \nabla_\theta \widehat{r}\left(s_t, a_t ; \theta_k\right)\right]-\mathbb{E}_{\tau^E \sim \pi_{k+1}}\left[\sum_{t \geq 0} \gamma^t \nabla_\theta \widehat{r}\left(s_t, a_t ; \theta_k\right)\right]\right\|\right]\\
& \quad+\alpha \mathbb{E}\left[\left\|\nabla \widehat{\mathcal{L}}\left(\theta_K\right)\right\|^2\right]-2 L_c L_q^2 \alpha^2,
\end{aligned}
\end{equation*}
where $(i)$ follows Lemma \ref{Lipschitz reward Lemma} and Cauchy-Schwartz inequality.

Therefore, we derive the following important relation by rearranging the inequality above:
\begin{equation}
\label{gradient_inequality_ori}
\begin{aligned}
    \alpha \mathbb{E}\left[\left\|\nabla \widehat{\mathcal{L}}\left(\theta_K\right)\right\|^2\right] &\leq 2 L_c L_q^2 \alpha^2+\mathbb{E}\left[\widehat{\mathcal{L}}\left(\theta_{k+1}\right)-\widehat{\mathcal{L}}\left(\theta_K\right)\right]+\mathcal{O}\left(m^{-1 / 2}\right)\\
    &+2 \alpha L_q \mathbb{E}\left[\left\|\mathbb{E}_{\tau^E \sim \pi_{\theta_k}}\left[\sum_{t \geq 0} \gamma^t \nabla_\theta \widehat{r}\left(s_t, a_t ; \theta_k\right)\right]-\mathbb{E}_{\tau^E \sim \pi_{k+1}}\left[\sum_{t \geq 0} \gamma^t \nabla_\theta \widehat{r}\left(s_t, a_t ; \theta_k\right)\right]\right\|\right].
\end{aligned}
\end{equation}

By invoking Lemma \ref{Lemma:sum_r_Q}, the following relation holds:
\begin{equation}
\label{gradient_inequality+Q}
    \alpha \mathbb{E}\left[\left\|\nabla \widehat{\mathcal{L}}\left(\theta_K\right)\right\|^2\right] \leq 2 L_c L_q^2 \alpha^2+\alpha C_1 \mathbb{E}\left[\left\|Q_{r_{\theta_k}, \pi_{\theta_k}}^{\text {soft }}-Q_{r_{\theta_k}, \pi_k}^{\text {sof }}\right\|_{2}\right]+\mathbb{E}\left[\widehat{\mathcal{L}}\left(\theta_{k+1}\right)-\widehat{\mathcal{L}}\left(\theta_K\right)\right]+\mathcal{O}\left(m^{-1 / 2}\right),
\end{equation}
where $C_1=4\alpha C_d L_q^2 $.

Summing the inequality \eqref{gradient_inequality+Q} from $k=0$ to $K-1$ and dividing both sides by $\alpha K$, it leads to
\begin{equation}
\label{reward_gradient_final}
\begin{aligned}
    \frac{1}{K} \sum_{k=0}^{K-1} \mathbb{E}\left[\left\|\nabla \widehat{\mathcal{L}}\left(\theta_K\right)\right\|^2\right]& \leq 2 L_c L_q^2 \alpha+\frac{C_1}{K} \sum_{k=0}^{K-1} \mathbb{E}\left[\left\|\widehat{Q}_{\widehat{r}_{\theta_k}, \pi_{\theta_k}}^{\text {soft }}-\widehat{Q}_{\widehat{r}_{\theta_k}, \pi_k}^{\text {soft }}\right\|_{2}\right]+\mathbb{E}\left[\frac{\widehat{\mathcal{L}}\left(\theta_K\right)-\widehat{\mathcal{L}}\left(\theta_0\right)}{K \alpha}\right]\\
    &\quad+\mathcal{O}\left(m^{-1 / 2}\right).
\end{aligned}
\end{equation}

\subsection{Theorem \ref{Theorem 1}}
The stepsize of reward update is set as $\alpha=\frac{\alpha_0}{K^\sigma}, \text { where } \sigma>0$. We plug \eqref{Alg2_Q_Bound} into \eqref{reward_gradient_final} and get the convergence result of reward parameter:
\begin{equation}
\begin{aligned}
     \frac{1}{K} \sum_{K=0}^{K-1} \mathbb{E}\left[\left\|\nabla \widehat{\mathcal{L}}\left(\theta_K\right)\right\|^2\right]&=\mathcal{O}\left(K^{-\sigma}\right)+\mathcal{O}\left(K^{-1+\sigma}\right)+\mathcal{O}\left(K^{-\frac{1}{4}}\right) \\&\quad+\mathcal{O}\left(m^{-1/2}+B^{3/2} m^{-1 /2}+B^{5 / 4} m^{-1 /4}\right).
\end{aligned}
\end{equation}

\subsection{Theorem \ref{Theorem 2}}
The stepsize of reward update is set as $\alpha=\frac{\alpha_0}{K^\sigma}, \text { where } \sigma>0$. We plug \eqref{soft Q final bound} into \eqref{reward_gradient_final} and get the convergence result of reward parameter:

\begin{equation}
\begin{aligned}
    \frac{1}{K} \sum_{K=0}^{K-1} \mathbb{E}\left[\left\|\nabla \widehat{\mathcal{L}}\left(\theta_K\right)\right\|^2\right]&=\mathcal{O}\left(K^{-\sigma}\right)+\mathcal{O}\left(K^{-1+\sigma}\right)+\mathcal{O}\left(K^{-\frac{1}{8}}\right) +\mathcal{O}\left(K^{-1+\sigma}\right)+\mathcal{O}\left(K^{-\frac{1}{4}}\right) \\
    &\quad +\mathcal{O}\left(m^{-1/2}+B^{3/2} m^{-1 / 4}+B^{5 / 4} m^{-1 / 8}\right).
\end{aligned}
\end{equation}

\section{Global Optimality of reward and policy parameters}
\label{Optimality_supp_proof}
 Now we consider the general case that reward is parametrized by neural networks.  Since $\widehat{\mathcal{L}}(\theta)$ is nonconcave in $\theta$, showing the global optimality of problem \eqref{empirical formulation} is non-trivial, it means that the stationary point of $\widehat{\mathcal{L}}(\theta)$ is not necessarily the global optimum. Therefore, we translate this problem to an equivalent saddle point problem:
\begin{equation*}
 \max _\theta \min _\pi \widehat{L}(\theta, \pi):=\mathbb{E}_{\tau^{\mathrm{E}} \sim\mathcal{D}}\left[\sum_{t=0}^{\infty} \gamma^t\left(\widehat{r}\left(s_t, a_t ; \theta\right))\right)\right]-\mathbb{E}_{\tau^{\mathrm{A}} \sim \pi}\left[\sum_{t=0}^{\infty} \gamma^t\left(\widehat{r}\left(s_t, a_t ; \theta\right)+\mathcal{H}\left(\pi\left(\cdot \mid s_t\right)\right)\right)\right] .
\end{equation*}

Based on the equivalence between \eqref{empirical formulation} and \eqref{bi-minmax} , to show the global optimality of reward and policy parameters, we first show that any stationary point $\tilde{\theta}$ in \eqref{empirical formulation} together with its corresponding optimal policy $\pi_{\tilde{\theta}}$ consist of a saddle point  $\left(\tilde{\theta}, \pi_{\tilde{\theta}}\right)$ to the problem \eqref{bi-minmax}  when the reward is overparameterized. Later, we illustrate that the reward parameter $\tilde{\theta}$ of the saddle point is the global maximizer of the empirical bi-level probelm \eqref{empirical formulation}. Recall that a tuple $\left(\tilde{\theta}, \pi_{\tilde{\theta}}\right)$ is called a saddle point of $\widehat{L}(\cdot, \cdot)$ if the following condition holds:
\begin{equation}
    \label{saddle point}
    \widehat{L}\left(\theta, \pi_{\tilde{\theta}}\right) \leq \widehat{L}\left(\tilde{\theta}, \pi_{\tilde{\theta}}\right) \leq \widehat{L}(\tilde{\theta}, \pi)
\end{equation}
for any other reward parameter $\theta$ and policy $\pi$. To show that $\left(\tilde{\theta}, \pi_{\tilde{\theta}}\right)$ satisfies the condition \eqref{saddle point}, we prove the following conditions respectively:
\begin{equation}
\label{reward saddle}
    \tilde{\theta} \in \arg \max _\theta \widehat{L}\left(\theta, \pi_{\tilde{\theta}}\right),
\end{equation}

\begin{equation}
    \label{policy saddle}
    \pi_{\tilde{\theta}} \in \arg \min _\pi \widehat{L}(\tilde{\theta}, \pi) .
\end{equation}
Here, we first show that any stationary point $\tilde{\theta}$ of the objective ${\widehat{\mathcal{L}}}(\cdot)$ in \eqref{empirical formulation} satisfies the optimality condition \eqref{reward saddle}. By the first-order condition, any stationary point $\tilde{\theta}$ of the objective ${\widehat{\mathcal{L}}}(\cdot)$ satisfies:
\begin{equation}
\label{first order condition}
    \nabla \widehat{\mathcal{L}}(\tilde{\theta})= \mathbb{E}_{\tau^E \sim \mathcal{D}}\left[\sum_{t \geq 0} \gamma^t \nabla_\theta \widehat{r}\left(s_t, a_t ; \tilde{\theta}\right)\right]-\mathbb{E}_{\tau^{\mathrm{A}} \sim \pi_\theta}\left[\sum_{t \geq 0} \gamma^t \nabla_\theta \widehat{r}\left(s_t, a_t ; \tilde{\theta}\right)\right]=0.
\end{equation}


Then we can look back at the minimax formulation $\widehat{L}(\cdot, \cdot)$ in \eqref{new_formulation}. Given any fixed policy $\pi$, we can write down the gradient of $\widehat{L}(\pi, \theta)$ w.r.t. the reward parameter $\theta$ explicitly as below:
\[
    \nabla_\theta  \widehat{L}(\pi_\theta, \theta)=\mathbb{E}_{\tau^E \sim \mathcal{D}}\left[\sum_{t \geq 0} \gamma^t \nabla_\theta \widehat{r}\left(s_t, a_t ; \tilde{\theta}\right)\right]-\mathbb{E}_{\tau^{\mathrm{A}} \sim \pi_\theta}\left[\sum_{t \geq 0} \gamma^t \nabla_\theta \widehat{r}\left(s_t, a_t ; \tilde{\theta}\right)\right].
\]

From the formulation, we could notice that $\nabla_\theta \widehat{L}\left(\theta_k,\pi_{\theta_k}\right)=\nabla \widehat{L}(\theta_k)$. Therefore, Proposition \ref{objective smoothness} can be generalized to the Theorem \ref{Thm: concavity}.

Due to the first-order condition we show in \eqref{first order condition}, we obtain the following first-order optimality result for the minimax objective $\widehat{L}(\cdot,\cdot)$:
\begin{equation}
\label{first_order_saddle}
    \nabla_\theta \widehat{L}\left(\theta=\tilde{\theta}, \pi=\pi_{\tilde{\theta}}\right)=0.
\end{equation}
 From Theorem \ref{Thm: concavity}, we have shown $\widehat{L}(\cdot, \cdot)$ is approximately concave in terms of the reward parameter $\theta$ given any fixed policy $\pi$, if the neural network is sufficiently large. With the concavity property in \eqref{new_formulation} and the condition in \eqref{first_order_saddle}, we have completed the proof of the first condition \eqref{reward saddle}.

Next, we move to prove the second condition \eqref{policy saddle}. Recall that $\pi_{\tilde{\theta}}$ is the optimal policy defined in \eqref{empirical formulation} under the reward parameter $\theta$. By observing the objective $\widehat{L}(\cdot, \cdot)$, we obtain the following result:
\begin{equation}
 \label{L_pi_division}   
\widehat{L}(\theta, \pi):=  \underbrace{\mathbb{E}_{\tau^{\mathrm{E}} \sim\left(\eta, \mathcal{D}\right)}\left[\sum_{t=0}^{\infty} \gamma^t\left(\widehat{r}\left(s_t, a_t ; \theta\right)\right)\right]}_{\text {Term } \mathrm{I}_1 \text {: independent of } \pi} -\underbrace{\mathbb{E}_{\tau^{\mathrm{A}} \sim(\eta, \pi)}\left[\sum_{t=0}^{\infty} \gamma^t\left(\widehat{r}\left(s_t, a_t ; \theta\right)+\mathcal{H}\left(\pi\left(\cdot \mid s_t\right)\right)\right)\right]}_{\text {Term } \mathrm{I}_2 \text { : a function of the policy } \pi} .
\end{equation}
The second term in \eqref{L_pi_division} coincides with the lower-level objective in \eqref{empirical formulation}. Moreover, since $\pi_{\tilde{\theta}}$ is the optimal policy defined in \eqref{empirical formulation} under reward parameter $\tilde{\theta}$, the following statement naturally holds:
\begin{equation*}
    \pi_{\tilde{\theta}} \in \arg \min _\pi \widehat{L}(\tilde{\theta}, \pi).
\end{equation*}
We now complete the proof to show the second condition \eqref{policy saddle}. Hence, we prove that any stationary point $\tilde{\theta}$ of the bi-level objective $\widehat{\mathcal{L}}(\cdot)$ together with its optimal policy $\pi_{{\tilde{\theta}}}$ is a saddle point of $\widehat{L}(\cdot,\cdot)$

Given a saddle point $\left(\tilde{\theta}, \pi_{\tilde{\theta}}\right)$ of $\widehat{L}(\cdot, \cdot)$, we have the following property that:
\begin{equation}
    \label{saddle inequality}
    \min _\pi \max _\theta \widehat{L}(\theta, \pi) \leq \max _\theta \widehat{L}\left(\theta, \pi_{\tilde{\theta}}\right) \stackrel{(i)}{=} \widehat{L}(\tilde{\theta}, \pi_{\tilde{\theta}})\stackrel{(i i)}{=} \min _\pi \widehat{L}(\tilde{\theta}, \pi) \leq \max _\theta \min _\pi \widehat{L}(\theta, \pi),
\end{equation}

where (i) follows the optimality condition \eqref{reward saddle} and (ii) follows the optimality condition \eqref{policy saddle}. According to the minimax inequality, we always have the following condition that
\begin{equation}
    \label{minimax_inequality}
    \max _\theta \min _\pi {L}(\theta, \pi) \leq \min _\pi \max _\theta {L}(\theta, \pi) .
\end{equation}

Putting the saddle point inequality \eqref{saddle inequality} and the minimax inequality \eqref{minimax_inequality} together, the following equality holds:
$$
\min _\pi \max _\theta \widehat{L}(\theta, \pi)=\max _\theta \widehat{L}\left(\theta, \pi_{\tilde{\theta}}\right)=\widehat{L}\left(\tilde{\theta}, \pi_{\tilde{\theta}}\right)=\min _\pi \widehat{L}(\tilde{\theta}, \pi)=\max _\theta \min _\pi \widehat{L}(\theta, \pi) .
$$

Therefore, for any saddle point $\left(\tilde{\theta}, \pi_{\tilde{\theta}}\right)$, the reward parameter $\tilde{\theta}$ and the corresponding policy $\pi_{\tilde{\theta}}$ satisfy the following:

\begin{align}
& \tilde{\theta} \in \arg \max _\theta \min _\pi \widehat{L}(\theta, \pi), \label{reward optimality}\\
& \pi_{\tilde{\theta}} \in \arg \min _\pi \max _\theta \widehat{L}(\theta, \pi). \label{policy optimality}
\end{align}

Due to the expression of the bi-level objective ${\widehat{\mathcal{L}}}(\cdot)$ in \eqref{empirical formulation} and the objective $\widehat{L}(\cdot, \cdot)$ in \eqref{new_formulation}, we have the following equality relationship for any reward parameter $\theta$ :
\begin{equation}
\label{L_L_min}
    \widehat{\mathcal{L}}(\theta)=\min _\pi \widehat{L}(\theta, \pi).
\end{equation}

Combining \eqref{reward optimality} and \eqref{L_L_min}, we yield the following result:
\begin{equation*}
\label{final_opt}
    \tilde{\theta} \in \arg \max _\theta \min _\pi \widehat{L}(\theta, \pi)=\arg \max _\theta \widehat{\mathcal{L}}(\theta).
\end{equation*}

Till now, we prove that for any saddle point $\left(\tilde{\theta}, \pi_{\tilde{\theta}}\right)$ of $\widehat{L}(\cdot, \cdot)$, the reward parameter $\tilde{\theta}$ constructs a globally optimal solution of the bi-level objective ${\widehat{\mathcal{L}}}(\cdot)$ in \eqref{empirical formulation}, when the neural network is overparameterized.

\vfill
